\documentclass[11pt]{article}

\usepackage{amsmath, amssymb, amsthm}
\usepackage{mathtools}
\usepackage{booktabs}
\usepackage{array}
\usepackage{hyperref}
\usepackage[round,authoryear]{natbib}
\usepackage{geometry}
\usepackage{multirow}
\usepackage{graphicx}
\usepackage{placeins}
\usepackage[title]{appendix}
\newtheorem{theorem}{Theorem}[section]
\newtheorem{lemma}[theorem]{Lemma}
\newtheorem{proposition}[theorem]{Proposition}
\newtheorem{corollary}[theorem]{Corollary}
\newtheorem{definition}[theorem]{Definition}
\newtheorem*{definition*}{Definition}
\newtheorem{example}[theorem]{Example}
\theoremstyle{remark}
\newtheorem{remark}[theorem]{Remark}
\newtheorem{assumption}[theorem]{Assumption}

\newcommand{\R}{\mathbb{R}}
\newcommand{\E}{\mathbb{E}}
\newcommand{\Tr}{\mathrm{Tr}}
\newcommand{\op}{\mathrm{op}}

\newcommand{\conv}{\mathrm{conv}}

\newcommand{\lmax}{\lambda_{\max}}

\begin{document}

\title{\textbf{Fisher Width: A Geometric Measure of Complexity\\
on Statistical Manifolds}}

\author{
Vu Khac Ky\\[4pt]
\small Department of Mathematics, FPT University, Vietnam\\
\small \texttt{kyvk2@fe.edu.vn}
}

\date{\today}
\maketitle


\begin{abstract}
Gaussian width is a central geometric complexity measure in high-dimensional probability, compressed sensing, convex optimization, and learning theory. It quantifies the average extent of a set along random directions, thereby capturing the effective dimension of constraint sets, hypothesis classes, and descent cones. However, this notion is intrinsically Euclidean. Statistical models instead carry a natural Riemannian geometry induced by the Fisher information metric, where directions are scaled according to statistical distinguishability rather than ambient Euclidean length.

We introduce Fisher width, a Fisher-geometric analogue of Gaussian width for statistical manifolds. At a parameter point $\theta$, Fisher width replaces the Euclidean identity by the local metric tensor $G(\theta)^{1/2}$, measuring the Gaussian width of the Fisher-rescaled set. This makes the resulting quantity sensitive to local statistical curvature and invariant under smooth reparameterizations.

We develop the basic theory of Fisher width, showing that it retains key structural features of Gaussian width, including concentration, metric perturbation stability, and spectral comparison bounds with the Euclidean baseline, while also capturing anisotropic geometric effects invisible to Euclidean measures. As an application, we prove a generalization bound for Fisher-Lipschitz hypothesis classes and propose computable estimators, which we evaluate empirically on MNIST across three model classes.

Fisher width is to statistical manifolds what Gaussian width is to Euclidean convex bodies. This work lays the foundation for studying complexity and learning on curved statistical manifolds.
\end{abstract}

\section{Introduction}
\label{sec:intro}

\subsection{Complexity on Statistical Manifolds}

A central problem in high-dimensional geometry and statistical learning is to quantify the effective complexity of a model class. One of the most successful notions for this purpose is the Gaussian width
\[
w(T) = \E_{g \sim \mathcal{N}(0,I_d)}
\!\left[
\sup_{v \in T}\langle g, v\rangle
\right],
\]
which plays a fundamental role in compressed sensing
\cite{Chandrasekaran2012}, phase transitions in convex recovery \cite{Amelunxen2014}, and empirical process theory
\cite{Bartlett2002}. Here \(T\) typically represents a geometric object encoding the complexity of a problem, such as a hypothesis class, a tangent cone, or a feasible region. Its power rests on a simple geometric insight: the average extent of \(T \subset \R^d\) in random directions captures its effective size in the ambient Euclidean space.

Modern statistical models, however, are rarely governed by Euclidean geometry alone. The parameters of an exponential family, a neural network, or a variational autoencoder naturally form a \emph{statistical manifold} $(\Theta, g_F)$ endowed with the Fisher information metric
\[
G(\theta)_{ij}
= \E_{x \sim p_\theta}\!\left[
  \frac{\partial \log p_\theta(x)}{\partial \theta_i}
  \frac{\partial \log p_\theta(x)}{\partial \theta_j}
\right].
\]
On such manifolds, statistical distinguishability is governed by the local Fisher metric rather than by ambient Euclidean distance. A unit displacement in parameter space may change the model dramatically in one direction and leave it almost unchanged in another, depending entirely on the local geometry encoded by $G(\theta)$. Classical Gaussian width, which treats all directions as equally significant, does not capture this information-geometric structure. This observation motivates the search for a notion of geometric complexity that is intrinsic to Fisher geometry in the same way that Gaussian width is intrinsic to Euclidean geometry.

\begin{definition*}[Fisher Width]
Let \(\theta_0\in\Theta\) and \(T\subset\mathbb R^d\). The
\emph{Fisher width} of \(T\) at \(\theta_0\) is
\[
w_G(T;\theta_0)
=
\E_{g\sim\mathcal N(0,I_d)}
\left[
\sup_{v\in T}
\langle g,G(\theta_0)^{1/2}v\rangle
\right].
\]
\end{definition*}

The factor $G(\theta_0)^{1/2}$ reweights directions by their local Fisher curvature, so that statistically sensitive directions contribute more to the width than insensitive ones. Thus Fisher width is a local quantity: it depends on the base point \(\theta_0\) and varies across the statistical manifold. When $G(\theta_0) = I_d$, it reduces to the classical Gaussian width.

\subsection{Main Contributions}
\label{sec:contributions}

The main contributions of this paper are as follows.

\begin{itemize}

\item[(C1)] We introduce Fisher width as a local Fisher-geometric analogue of Gaussian width on statistical manifolds. We establish the lifting identity
\[
w_G(T;\theta)=w(G(\theta)^{1/2}T),
\]
and prove its invariance under smooth reparameterizations.

\item[(C2)] We develop the structural theory of Fisher width, including concentration inequalities, algebraic properties, spectral comparison bounds, and stability under metric perturbations. We further prove an empirical Fisher
stability theorem, showing that Fisher width can be consistently approximated from score samples under operator-norm concentration of the empirical Fisher
matrix.

\item[(C3)] We prove a generalization bound for Fisher-Lipschitz hypothesis classes, showing that the uniform deviation is controlled by
\[
\frac{w_G(T-T;\theta_0)}{\sqrt{n}}.
\]
We also show that this scale is sharp, up to constants, for local exponential-family likelihood models. Thus Fisher width appears in Fisher-geometric learning bounds in a role analogous to Gaussian width and Rademacher complexity in Euclidean settings.

\item[(C4)] We develop practical estimators based on empirical Fisher information, randomized low-rank approximation, and score-based sampling, and validate these estimators on MNIST across logistic regression, softmax regression, and ridge regression models.

\end{itemize}

\subsection{Related Work}
\label{sec:related}

\paragraph{Gaussian width and high-dimensional geometry.}
Gaussian width is a classical complexity measure in asymptotic convex geometry and
high-dimensional probability. Gordon's escape-through-a-mesh theorem
\cite{Gordon1988} established its central role in the geometry of random projections.
It has since become a standard tool in compressed sensing and convex recovery
\cite{Chandrasekaran2012}, where sharp phase transitions are often described through
the closely related statistical dimension of convex cones \cite{Amelunxen2014}.
Gaussian-process and chaining methods provide a complementary analytic viewpoint,
beginning with Dudley's entropy bound \cite{Dudley1967} and continuing through the
modern high-dimensional probability literature
\cite{LedouxTalagrand1991,Vershynin2018,Wainwright2019}. Our work follows this
geometric tradition, but replaces the ambient Euclidean metric by the local Fisher
metric induced by a statistical model.

\paragraph{Complexity measures in statistical learning.}
Gaussian and Rademacher complexities are fundamental measures of hypothesis-class
richness in learning theory \cite{Bartlett2002}. Localized empirical-process methods
and oracle inequalities further refine these complexity measures by adapting them to
the geometry of the function class and the data distribution \cite{Koltchinskii2011}.
More recent work has developed norm- and architecture-dependent complexity bounds for
neural networks, including size-independent Rademacher bounds \cite{Golowich2018}.
Fisher width belongs to this family of width-type complexity measures: it is a
functional controlling uniform deviations, but the underlying geometry is Fisher
rather than Euclidean.

\paragraph{Information geometry and Fisher metrics.}
The Fisher information metric is the canonical Riemannian metric on many statistical
model spaces. It underlies the classical theory of statistical manifolds, dual
connections, and information-geometric curvature
\cite{AmariNagaoka2000,Amari2016,Ay2017,MurrayRice1993}. Statistical curvature was
introduced by Efron \cite{Efron1975}, and the local metric expansions used in this
paper are standard in Riemannian geometry \cite{doCarmo1992}. Existing work in
information geometry primarily studies divergences, geodesics, projections, curvature,
and optimization. In contrast, Fisher width introduces a Gaussian-width-type
complexity functional on statistical manifolds, thereby connecting information
geometry with high-dimensional geometric complexity.

\paragraph{Fisher geometry in optimization and deep learning.}
Fisher geometry has long been used in optimization through natural gradient descent
\cite{Amari1998}. In modern machine learning, it appears in Riemannian metrics for
neural networks \cite{Ollivier2015}, scalable curvature approximations such as K-FAC
\cite{MartensGrosse2015}, and analyses of natural-gradient methods
\cite{Martens2020}. The empirical Fisher is computationally convenient but is not, in
general, equivalent to the population Fisher or the Hessian, and can exhibit distinct
limitations \cite{Kunstner2019}. The spectrum of the Fisher information matrix in
deep networks can also be highly anisotropic, with a small number of large directions
and many near-flat directions \cite{Karakida2019}. Fisher width is complementary to
this optimization literature: instead of studying the dynamics of an algorithm, it
uses the Fisher metric to measure the effective geometric size of a set.

\paragraph{Fisher--Rao complexity and related generalization perspectives.}
The closest prior work is the Fisher--Rao norm of neural networks
\cite{Liang2019}, which uses the Fisher metric to define a parameter-space complexity
measure. The distinction is structural: the Fisher--Rao norm measures the length of a
single parameter vector, whereas Fisher width measures the size of an entire set after
Fisher deformation. Thus Fisher width is a set-complexity measure in the tradition of
Gaussian width, rather than a parameter norm.

PAC-Bayes theory provides another information-sensitive approach to generalization,
relating risk bounds to divergences between posterior and prior distributions
\cite{McAllester1999,Seeger2002,DziugaiteRoy2017}; see \cite{Alquier2024} for a
recent survey. For exponential-family models, such divergences admit local quadratic
approximations involving the Fisher metric. Flatness-based explanations of
generalization are also related, beginning with flat minima
\cite{HochreiterSchmidhuber1997} and continuing through sharp-minima phenomena
\cite{Keskar2017}, sharpness-aware optimization \cite{Foret2021}, and empirical
studies of generalization measures \cite{Jiang2020}. Fisher width provides a
different but compatible viewpoint: it measures the Fisher-geometric size of a class
or perturbation set, rather than the sharpness of a single trained solution.

To the best of our knowledge, no prior work has introduced a Gaussian-width-type
complexity functional on statistical manifolds or developed its structural,
statistical, and computational properties under the Fisher information metric.

\section{Fisher Width}
\label{sec:definition}

\subsection{Statistical Manifolds and Fisher Geometry}
\label{sec:fisher-geometry}

Let $\mathcal{X}$ be a sample space and $\mathcal{P} = \{p_\theta : \theta \in \Theta\}$ a smooth parametric
family with $\Theta \subset \mathbb{R}^d$. The \emph{Fisher information matrix} at $\theta$ is
\[
G(\theta)
= \mathbb{E}_{x \sim p_\theta}\!\left[
  \nabla_\theta \log p_\theta(x)\,
  \nabla_\theta \log p_\theta(x)^\top
\right].
\]
We assume $G(\theta) \succ 0$ throughout. The \emph{Fisher metric} $g_F(u,v)|_\theta = u^\top G(\theta) v$ makes $\Theta$ a Riemannian manifold. The Fisher metric characterizes the local statistical distinguishability of nearby distributions \cite{AmariNagaoka2000,Amari2016}. Specifically, for sufficiently small $\Delta\theta$, the Kullback--Leibler divergence
\[
D_{\mathrm{KL}}(p\|q)
=
\int p(x)\log\frac{p(x)}{q(x)}\,dx
\]
admits the local expansion
\[
D_{\mathrm{KL}}\!\left(
  p_\theta \;\|\; p_{\theta+\Delta\theta}
\right)
=
\frac{1}{2}
\Delta\theta^\top G(\theta)\,\Delta\theta
+ o(\|\Delta\theta\|^2)
\]
(see \citep[Chapter 2]{AmariNagaoka2000}). Thus, the Fisher metric provides the local quadratic approximation of statistical distinguishability independently of the ambient Euclidean structure on $\mathbb{R}^d$.

\subsection{Fisher Width}
\label{sec:fisher-width}

We begin by recalling the classical Gaussian width. For a set $T \subset \mathbb{R}^d$, the \emph{Gaussian width} is defined by
\[
w(T)
= \mathbb{E}_{g \sim \mathcal{N}(0,I_d)}\!\left[
  \sup_{v \in T}\langle g, v\rangle\right].
\]

The following standard properties will serve as reference points for the Fisher-width theory developed later.

\begin{proposition}[Basic Properties of Gaussian Width]
\label{prop:gw-basic}
For compact sets $T, T_1, T_2 \subset \mathbb{R}^d$ and
scalar $a \in \mathbb{R}$:
\begin{enumerate}
\item If $T_1 \subseteq T_2$, then $w(T_1) \le w(T_2)$.
\item $w(aT) = |a|\,w(T)$.
\item $w(\operatorname{conv}(T)) = w(T)$.
\item $w(T_1 + T_2) \le w(T_1) + w(T_2)$.
\end{enumerate}
\end{proposition}

We now introduce Fisher width, the Fisher-geometric analogue of Gaussian width.

\begin{definition}[Fisher Width]
\label{def:fisher-width}
Let $\theta_0 \in \Theta$ and $G(\theta_0) \succ 0$.
The \emph{Fisher width} of $T \subset \mathbb{R}^d$ at $\theta_0$ is
\[
w_G(T;\theta_0)
=
\mathbb{E}_{g\sim\mathcal N(0,I_d)}
\!\left[
\sup_{v\in T}
\langle g,\,G(\theta_0)^{1/2}v\rangle
\right].
\]
When $\theta_0$ is fixed or clear from context, we simply write
$w_G(T)$.
\end{definition}

The map
\[
v \longmapsto G(\theta_0)^{1/2}v
\]
rescales directions according to the local Fisher geometry. Directions with large
Fisher scaling contribute more strongly to the width, while statistically insensitive
directions contribute less. Consequently, Fisher width measures the effective size
of \(T\) after reweighting directions by their local statistical distinguishability.

When \(G(\theta_0)=I_d\), Fisher width reduces to the classical Gaussian width:
\[
w_G(T;\theta_0)=w(T).
\]
In general, \(w_G(T;\theta)\) varies across the statistical manifold, in contrast to
Gaussian width, which depends only on the set \(T\) and not on any base point.

\subsection{Equivalent Forms}
\label{sec:equiv}

\begin{proposition}[Lifting Identity]
\label{prop:lifting-identity}

Let \(T\subset\mathbb R^d\) be compact and let \(G\succ0\).
Then
\[
w_G(T)
=
w\!\left(G^{1/2}T\right), \quad \text{where} \quad 
G^{1/2}T
=
\{G^{1/2}v:\;v\in T\}.
\]
\end{proposition}

\begin{proof}
By definition,
\[
w_G(T)
=
\mathbb E_g
\left[
\sup_{v\in T}
\langle g,G^{1/2}v\rangle
\right].
\]
Since \(u=G^{1/2}v\) ranges over \(G^{1/2}T\), we have
\[
\sup_{v\in T}
\langle g,G^{1/2}v\rangle
=
\sup_{u\in G^{1/2}T}
\langle g,u\rangle.
\]
Therefore
\[
w_G(T)
=
\mathbb E_g
\left[
\sup_{u\in G^{1/2}T}
\langle g,u\rangle
\right]
=
w\!\left(G^{1/2}T\right).
\]
\end{proof}

The lifting identity
\[
w_G(T;\theta_0)=w\!\left(G(\theta_0)^{1/2}T\right)
\]
is the central structural observation underlying the theory. It shows that, at a fixed
base point, Fisher width is precisely the Gaussian width of the Fisher-rescaled set.
Consequently, many classical properties of Gaussian width can be transferred to the
Fisher setting through the local metric deformation induced by the Fisher information
matrix.

The geometric content arises because this deformation depends on the base point:
\[
T \longmapsto G(\theta)^{1/2}T .
\]
Thus, as \(\theta\) varies across the statistical manifold, the same Euclidean set
\(T\) may acquire different effective widths under the local information geometry.
Fisher width therefore measures not only the shape of \(T\), but also how that shape is
seen through the Fisher geometry of the model.

\subsection{Examples}
\label{sec:examples}

\begin{example}[Euclidean Ball]
\label{ex:euclidean-ball}
Let 
\[
T = rB_2^d = \{v \in \mathbb{R}^d : \|v\|_2 \le r\}.
\]
By the Cauchy--Schwarz inequality,
\[
\langle g, G^{1/2}v \rangle
= \langle G^{1/2}g, v \rangle
\le \|G^{1/2}g\|_2\,\|v\|_2,
\]
with equality when $v$ is parallel to $G^{1/2}g$. Therefore
\[
\sup_{\|v\|_2 \le r}
\langle g, G^{1/2}v \rangle
= r\,\|G^{1/2}g\|_2,
\]
and hence $w_G(T) = r\,\mathbb{E}\|G^{1/2}g\|_2$.
Since \(g\sim N(0,I_d)\), we have
\[
\mathbb E[gg^\top]=I_d.
\]
Thus
\[
\mathbb E\|G^{1/2}g\|_2^2
=
\mathbb E[g^\top Gg]
=
\operatorname{Tr}\!\left(G\,\mathbb E[gg^\top]\right)
=
\operatorname{Tr}(G).
\]
By Jensen's inequality,
\[
\mathbb E\|G^{1/2}g\|_2
=
\mathbb E\sqrt{\|G^{1/2}g\|_2^2}
\le
\sqrt{
\mathbb E\|G^{1/2}g\|_2^2
}
=
\sqrt{\operatorname{Tr}(G)}.
\]
Consequently,
\[
w_G(rB_2^d)
\le
r\sqrt{\operatorname{Tr}(G)}.
\]

Thus the Fisher width of a Euclidean ball is controlled by the total Fisher
information \(\operatorname{Tr}(G)\).
\end{example}

\begin{example}[Sparse Vectors]
Consider
\[
T = \{v \in \mathbb{R}^d : \|v\|_0 \le s,\ \|v\|_2 \le 1\},
\]
the set of $s$-sparse unit vectors. By Cauchy--Schwarz, the supremum
$\sup_{v \in T}\langle u, v\rangle$ for any fixed $u \in \mathbb{R}^d$
is attained by concentrating the support of $v$ on the $s$ largest
coordinates of $u$, giving
\[
\sup_{\|v\|_0 \le s,\,\|v\|_2 \le 1} \langle u, v \rangle
= \|u\|_{(s)}
:= \Bigl(\sum_{i=1}^s u_{(i)}^2\Bigr)^{1/2},
\]
where $|u_{(1)}| \ge \cdots \ge |u_{(d)}|$. Applying this with $u = G^{1/2}g$,
\[
w_G(T) = \mathbb{E}\,\|G^{1/2}g\|_{(s)}.
\]
When $G = I_d$ this reduces to $w(T) = \mathbb{E}\,\|g\|_{(s)}$.
The Fisher metric replaces the isotropic vector $g$ by the anisotropically
weighted $G^{1/2}g$, so two equally sparse parameter vectors may contribute
very differently to the width if they are aligned with directions carrying
different amounts of Fisher information.
\end{example}

\begin{example}[Exponential Families]
Consider an exponential family
\[
p_\theta(x)
=
h(x)
\exp\!\bigl(
\theta^\top\phi(x)-A(\theta)
\bigr),
\]
with sufficient statistic $\phi(x)$. Its Fisher information matrix is
\[
G(\theta)
=
\nabla^2 A(\theta)
=
\mathbb{E}_{p_\theta}\!\left[
\bigl(\phi(x)-\mathbb{E}_{p_\theta}[\phi(x)]\bigr)
\bigl(\phi(x)-\mathbb{E}_{p_\theta}[\phi(x)]\bigr)^\top
\right]
\]
\citep{AmariNagaoka2000}.
Thus the Fisher matrix, and hence Fisher width, is determined entirely by the covariance structure of the sufficient statistics. For the Euclidean ball $T = rB_2^d$,
\[
w_G(T)
\le r\sqrt{\mathrm{Tr}(G(\theta))}
= r\sqrt{\mathrm{Tr}(\mathrm{Cov}_{p_\theta}[\phi(x)])},
\]
the square root of the total variance of the sufficient statistic under $p_\theta$.
Fisher width is thus bounded above by a quantity that concentrates on directions
of high statistical curvature; directions along which $\phi(x)$ is nearly
deterministic under $p_\theta$ contribute little to this bound.
\end{example}

\begin{example}[Linear subspaces]
\label{ex:subspace}

Let \(V \subset \mathbb{R}^d\) be a \(k\)-dimensional linear subspace and let
\(T = V \cap B_2^d\). Denote by \(P_V\) the orthogonal projector onto \(V\). Since
\[
\sup_{\substack{v \in V \\ \|v\|_2 \le 1}}
\langle g, G^{1/2}v \rangle = \|P_V G^{1/2}g\|_2
\]
we have
\[
w_G(T)=\mathbb{E}\|P_VG^{1/2}g\|_2 
 \le
\sqrt{\mathbb{E}\|P_VG^{1/2}g\|_2^2} \quad \text{(by Jensen's inequality)}.
\]

Note first that
\[
\|P_VG^{1/2}g\|_2^2
=
(P_VG^{1/2}g)^\top(P_VG^{1/2}g)
=
g^\top G^{1/2}P_V^\top P_VG^{1/2}g
=
g^\top G^{1/2}P_VG^{1/2}g,
\]
where we use \(P_V^2=P_V = P_V^\top\). Hence
\[
\begin{aligned}
\mathbb{E}\|P_VG^{1/2}g\|_2^2
&=
\mathbb{E}\bigl[g^\top G^{1/2}P_VG^{1/2}g\bigr] \\
&=
\operatorname{Tr}(G^{1/2}P_VG^{1/2}) \\
&=
\operatorname{Tr}(GP_V) \\
&=
\operatorname{Tr}(P_VGP_V).
\end{aligned}
\]
Here the second identity uses the quadratic-form trace identity
\(\mathbb{E}[g^\top A g]=\operatorname{Tr}(A)\) for \(g\sim\mathcal N(0,I_d)\),
while the last two identities follow from cyclicity of the trace and \(P_V^2=P_V\). Therefore
\[
w_G(T)
\le
\sqrt{\operatorname{Tr}(P_VGP_V)}.
\]
Equivalently, for any orthonormal basis \(\{u_1,\ldots,u_k\}\) of \(V\),
\[
\operatorname{Tr}(P_VGP_V)
=
\sum_{i=1}^k u_i^\top G u_i.
\]
Thus Fisher width is controlled by the total Fisher information carried by the
directions of \(V\). This is the subspace analogue of the full-ball bound controlled
by \(\operatorname{Tr}(G)\): only the Fisher energy restricted to \(V\) contributes to
the width.
\end{example}

\subsection{Basic Properties}
\label{sec:basic-props}

\noindent \begin{lemma}[Monotonicity and Homogeneity]
\label{lem:mono-homog}
\begin{enumerate} \mbox{ }
  \item Let $T_1, T_2 \subset \mathbb{R}^d$ be compact. If $T_1 \subseteq T_2$, then
  \[
    w_G(T_1) \le w_G(T_2).
  \]
  \item Let $T \subset \mathbb{R}^d$ be compact and symmetric. Then for every
  $\alpha \in \mathbb{R}$,
  \[
    w_G(\alpha T) = |\alpha|\, w_G(T).
  \]
\end{enumerate}
\end{lemma}

The lemma follows directly from the definition.

\begin{lemma}[Algebraic properties]
\label{lem:algebraic}
Let \(G\succ 0\). For nonempty compact sets
\(T,T_1,T_2\subset\mathbb R^d\),
\begin{align*}
w_G(T_1 + T_2) &\leq w_G(T_1) + w_G(T_2), \\
w_G(\conv(T))  &= w_G(T), \\
w_G(T_1 \cap T_2) &\leq \min\{w_G(T_1), w_G(T_2)\}.
\end{align*}
\end{lemma}

\begin{proof}
By the lifting identity,
\[
w_G(T)=w(G^{1/2}T).
\]
Since \(G^{1/2}\) is linear and injective,
\[
G^{1/2}(T_1+T_2)
=
G^{1/2}T_1+G^{1/2}T_2,
\]
\[
G^{1/2}\operatorname{conv}(T)
=
\operatorname{conv}(G^{1/2}T),
\]
and
\[
G^{1/2}(T_1\cap T_2)
=
G^{1/2}T_1\cap G^{1/2}T_2.
\]
The claims follow from the corresponding algebraic properties of Gaussian width.
\end{proof}

\begin{lemma}[Spectral bounds]
\label{lem:spectral-bounds}
Let \(T\subset\mathbb R^d\) be compact. Then
\[
\sqrt{\lambda_{\min}(G)}\,w(T)
\le
w_G(T)
\le
\sqrt{\lambda_{\max}(G)}\,w(T).
\]
\end{lemma}

\begin{proof}
Since \(G^{1/2}\) is symmetric, for \(g\sim\mathcal N(0,I_d)\) we may write
\[
\langle g,G^{1/2}v\rangle
=
\langle G^{1/2}g,v\rangle.
\]
Thus, with \(h=G^{1/2}g\sim\mathcal N(0,G)\),
\[
w_G(T)
=
\mathbb E_h\sup_{v\in T}\langle h,v\rangle .
\]

Define the centered Gaussian processes
\[
X_v=\langle h,v\rangle,
\qquad
Y_v=\sqrt{\lambda_{\max}(G)}\,\langle g,v\rangle,
\qquad v\in T.
\]
For all \(v,v'\in T\),
\[
\mathbb E|X_v-X_{v'}|^2
=
(v-v')^\top G(v-v')
\le
\lambda_{\max}(G)\|v-v'\|_2^2
=
\mathbb E|Y_v-Y_{v'}|^2.
\]
By the Sudakov--Fernique comparison inequality \cite[Theorem 2.2.3]{talagrand2014},
\[
w_G(T)
=
\mathbb E\sup_{v\in T}X_v
\le
\mathbb E\sup_{v\in T}Y_v
=
\sqrt{\lambda_{\max}(G)}\,w(T).
\]

For the lower bound, define
\[
Z_v=\sqrt{\lambda_{\min}(G)}\,\langle g,v\rangle .
\]
Then
\[
\mathbb E|Z_v-Z_{v'}|^2
=
\lambda_{\min}(G)\|v-v'\|_2^2
\le
(v-v')^\top G(v-v')
=
\mathbb E|X_v-X_{v'}|^2.
\]
Applying Sudakov--Fernique again gives
\[
\sqrt{\lambda_{\min}(G)}\,w(T)
=
\mathbb E\sup_{v\in T}Z_v
\le
\mathbb E\sup_{v\in T}X_v
=
w_G(T).
\]
\end{proof}

\begin{theorem}[Gaussian Concentration]
\label{thm:concentration}
Let $R(T) = \sup_{v \in T}\|v\|_2 < \infty$. Define
\[
\varphi(g) = \sup_{v \in T}\langle g, G^{1/2}v\rangle.
\]
Then $\varphi$ is $\sqrt{\lmax(G)}\,R(T)$-Lipschitz and satisfies
\[
\Pr\!\left(\varphi(g) - w_G(T) > t\right)
\leq \exp\!\left(-\frac{t^2}{2\,\lmax(G)\,R(T)^2}\right).
\]
\end{theorem}

\begin{proof}
For $g, h \in \R^d$,
\begin{align*}
|\varphi(g) - \varphi(h)|
&\leq \sup_{v \in T}|\langle g - h,\, G^{1/2}v\rangle| \\
&\leq \|g - h\|_2 \sup_{v \in T}\|G^{1/2}v\|_2 \\
&\leq \sqrt{\lmax(G)}\,R(T)\,\|g - h\|_2.
\end{align*}
Thus $\varphi$ is $\sqrt{\lmax(G)}\,R(T)$-Lipschitz.
The concentration inequality follows from the Gaussian concentration
theorem \cite{LedouxTalagrand1991}.
\end{proof}

\begin{lemma}[Reparameterization invariance]
\label{lem:reparam}
Let \(\varphi:\widetilde\Theta\to\Theta\) be a smooth reparameterization with
invertible Jacobian
\[
J = J_\varphi(\tilde\theta_0).
\]
Let \(\theta_0=\varphi(\tilde\theta_0)\), let \(G=G(\theta_0)\), and let
\[
\widetilde G = J^\top GJ
\]
be the pullback Fisher metric at \(\tilde\theta_0\). For \(T\subset\mathbb R^d\),
define its pullback
\[
\widetilde T
=
J^{-1}T
=
\{\tilde v\in\mathbb R^d:\ J\tilde v\in T\}.
\]
Then
\[
w_{\widetilde G}(\widetilde T)=w_G(T).
\]
\end{lemma}

\begin{proof}
Let \(g\sim\mathcal N(0,I_d)\). Since
\[
\operatorname{Cov}(J^\top G^{1/2}g)
=
J^\top GJ
=
\widetilde G,
\]
the centered Gaussian vectors \(J^\top G^{1/2}g\) and
\(\widetilde G^{1/2}g\) have the same distribution. Hence
\[
\begin{aligned}
w_{\widetilde G}(\widetilde T)
&=
\mathbb E\sup_{\tilde v\in\widetilde T}
\langle \widetilde G^{1/2}g,\tilde v\rangle \\
&=
\mathbb E\sup_{\tilde v\in\widetilde T}
\langle J^\top G^{1/2}g,\tilde v\rangle \\
&=
\mathbb E\sup_{\tilde v:\,J\tilde v\in T}
\langle G^{1/2}g,J\tilde v\rangle \\
&=
\mathbb E\sup_{v\in T}
\langle G^{1/2}g,v\rangle \\
&=
w_G(T),
\end{aligned}
\]
where the fourth equality uses the bijection \(v=J\tilde v\) between
\(\widetilde T\) and \(T\).
\end{proof}

The equality is distributional rather than pointwise: in general,
\[
(J^\top GJ)^{1/2}\neq J^\top G^{1/2}.
\]
Thus Fisher width is invariant under reparameterization only when both the metric and
the tangent set are transformed by pullback.

\section{Structural Properties}
\label{sec:structural}

The identity $w_G(T) = w(G^{1/2}T)$ allows structural properties of Gaussian width to transfer directly to the Fisher setting. We focus on the stability under metric perturbations and subspace geometry.

\subsection{Estimation Stability}
\label{subsec:estimation-stability}

In practice, the Fisher matrix is rarely available exactly.
Instead, one typically works with empirical Fisher matrices, low-rank approximations, diagonal approximations, or Kronecker-factored approximations. A natural question is therefore whether Fisher width is stable under perturbations of the underlying metric. The following result shows that errors in the Fisher metric translate linearly into errors in Fisher width.

\begin{theorem}[Estimation stability]
\label{thm:estimation-stability}
Let \(G_1,G_2\succeq0\). If
\[
\|G_1^{1/2}-G_2^{1/2}\|_{\mathrm{op}}\le \varepsilon,
\]
then for every compact set \(T\subset\mathbb R^d\),
\[
\bigl|w_{G_1}(T)-w_{G_2}(T)\bigr|
\le
\varepsilon w(T).
\]
\end{theorem}

\begin{proof}
Let
\[
M=G_1^{1/2}-G_2^{1/2}.
\]
For each \(g\),
\[
\sup_{v\in T}\langle g,G_1^{1/2}v\rangle
-
\sup_{v\in T}\langle g,G_2^{1/2}v\rangle
\le
\sup_{v\in T}\langle g,Mv\rangle.
\]
Taking expectations gives
\[
w_{G_1}(T)-w_{G_2}(T)
\le
w(MT).
\]
We now bound \(w(MT)\) by Sudakov--Fernique. Consider the centered Gaussian
processes
\[
X_v=\langle g,Mv\rangle,
\qquad
Y_v=\varepsilon\langle g,v\rangle,
\qquad v\in T.
\]
Since \(\|M\|_{\mathrm{op}}\le\varepsilon\),
\[
\mathbb E|X_u-X_v|^2
=
\|M(u-v)\|_2^2
\le
\varepsilon^2\|u-v\|_2^2
=
\mathbb E|Y_u-Y_v|^2
\]
for all \(u,v\in T\). Hence, by the Sudakov--Fernique comparison inequality,
\[
w(MT)\le \varepsilon w(T).
\]
Therefore
\[
w_{G_1}(T)-w_{G_2}(T)
\le
\varepsilon w(T).
\]
Interchanging \(G_1\) and \(G_2\) gives the reverse bound
\[
w_{G_2}(T)-w_{G_1}(T)
\le
\varepsilon w(T).
\]
Combining the two inequalities yields
\[
\bigl|w_{G_1}(T)-w_{G_2}(T)\bigr|
\le
\varepsilon w(T).
\]
\end{proof}

We next apply the stability theorem to empirical Fisher matrices. Suppose that
the Fisher information at \(\theta_0\) is estimated from independent samples
\[
z_1,\ldots,z_n\sim p_{\theta_0}.
\]
Define the score vector by
\[
s_{\theta_0}(z)
=
\nabla_\theta \log p_\theta(z)\big|_{\theta=\theta_0},
\]
and write
\[
s_i=s_{\theta_0}(z_i),
\qquad i=1,\ldots,n.
\]
Under the standard regularity conditions for Fisher information,
\[
G=G(\theta_0)
=
\mathbb E_{z\sim p_{\theta_0}}
\bigl[
s_{\theta_0}(z)s_{\theta_0}(z)^\top
\bigr]
=
\mathbb E[s_i s_i^\top].
\]
The empirical Fisher matrix is then
\[
\widehat G_n
=
\frac1n\sum_{i=1}^n
s_i s_i^\top .
\]
We start with a concentration bound for
\(\widehat G_n\).

\begin{lemma}[Empirical Fisher concentration]
\label{lem:empirical-fisher-concentration}
Assume
\[
\|s_{\theta_0}(z)\|_2\le B
\qquad\text{almost surely}.
\]
Then, with probability at least \(1-\delta\),
\[
\|\widehat G_n-G\|_{\op}
\le
B\sqrt{
\frac{2\|G\|_{\op}\log(2d/\delta)}{n}
}
+
\frac{4B^2}{3}
\frac{\log(2d/\delta)}{n}.
\]
\end{lemma}

\begin{proof}
For each \(i\), write
\[
s_i=s_{\theta_0}(z_i),
\qquad
X_i=s_i s_i^\top-G.
\]
Then \(\mathbb E X_i=0\) and
\[
\widehat G_n-G=\frac1n\sum_{i=1}^n X_i.
\]

We first bound the almost-sure operator norm of \(X_i\). Since
\[
\|s_i s_i^\top\|_{\op}
=
\|s_i\|_2^2
\le B^2
\]
and
\[
\|G\|_{\op}
=
\bigl\|\mathbb E[s_i s_i^\top]\bigr\|_{\op}
\le
\mathbb E\|s_i s_i^\top\|_{\op}
\le
B^2,
\]
we have
\[
\|X_i\|_{\op}
\le
2B^2
\qquad\text{almost surely}.
\]

Next, we bound the variance proxy. Expanding \(X_i^2\) gives
\[
\begin{aligned}
\mathbb E X_i^2
&=
\mathbb E\bigl[(s_i s_i^\top-G)^2\bigr] \\
&=
\mathbb E[(s_i s_i^\top)^2]
-
\mathbb E[s_i s_i^\top]G
-
G\mathbb E[s_i s_i^\top]
+
G^2 \\
&=
\mathbb E[(s_i s_i^\top)^2]-G^2.
\end{aligned}
\]
Since \(G^2\succeq0\),
\[
\mathbb E X_i^2
\preceq
\mathbb E[(s_i s_i^\top)^2].
\]
Moreover,
\[
(s_i s_i^\top)^2
=
\|s_i\|_2^2 s_i s_i^\top
\preceq
B^2 s_i s_i^\top.
\]
Taking expectations yields
\[
\mathbb E X_i^2
\preceq
B^2\mathbb E[s_i s_i^\top]
=
B^2G.
\]
Therefore
\[
\left\|
\sum_{i=1}^n \mathbb E X_i^2
\right\|_{\op}
\le
nB^2\|G\|_{\op}.
\]

By the matrix Bernstein inequality
\citep[Theorem~6.1.1]{Tropp2015},
\[
\Pr\left(
\left\|\sum_{i=1}^n X_i\right\|_{\op}\ge t
\right)
\le
2d\exp\left(
-\frac{t^2/2}{\sigma^2+Lt/3}
\right),
\]
where
\[
L=2B^2,
\qquad
\sigma^2
=
\left\|
\sum_{i=1}^n \mathbb E X_i^2
\right\|_{\op}
\le
nB^2\|G\|_{\op}.
\]
Taking
\[
u=\log(\frac{2d}{\delta}), 
\qquad
t=\sqrt{2\sigma^2 u}+\frac{2}{3}Lu,
\]
gives, with probability at least \(1-\delta\),
\[
\left\|\sum_{i=1}^n X_i\right\|_{\op}
\le
\sqrt{2\sigma^2\log(\frac{2d}{\delta})}
+
\frac{2}{3}L\log(\frac{2d}{\delta}).
\]
Therefore
\[
\left\|\sum_{i=1}^n X_i\right\|_{\op}
\le
B\sqrt{2n\|G\|_{\op}\log(\frac{2d}{\delta})}
+
\frac{4}{3}B^2\log(\frac{2d}{\delta}).
\]
Dividing by \(n\), we obtain
\[
\|\widehat G_n-G\|_{\op}
\le
B\sqrt{
\frac{2\|G\|_{\op}\log(\frac{2d}{\delta})}{n}
}
+
\frac{4B^2}{3}
\frac{\log(\frac{2d}{\delta})}{n}.
\]
\end{proof}

\begin{lemma}[Square-root perturbation]
\label{lem:sqrt-perturbation-op}
Let \(A,B\succ0\). Suppose
\[
A\succeq aI,
\qquad
B\succeq bI
\]
for some \(a,b>0\). Then
\[
\|A^{1/2}-B^{1/2}\|_{\op}
\le
\frac{\|A-B\|_{\op}}{\sqrt a+\sqrt b}.
\]
\end{lemma}

\begin{proof}
We use the standard integral representation for the matrix square root
\citep[Chapter~V]{Bhatia1997}:
\[
A^{1/2}-B^{1/2}
=
\frac{1}{\pi}
\int_0^\infty
t^{1/2}
(A+tI)^{-1}
(A-B)
(B+tI)^{-1}
\,dt .
\]
Taking operator norms and using the triangle inequality for Bochner integrals gives
\[
\|A^{1/2}-B^{1/2}\|_{\op}
\le
\frac{1}{\pi}
\int_0^\infty
t^{1/2}
\left\|
(A+tI)^{-1}
(A-B)
(B+tI)^{-1}
\right\|_{\op}
\,dt .
\]
By submultiplicativity of the operator norm,
\[
\left\|
(A+tI)^{-1}
(A-B)
(B+tI)^{-1}
\right\|_{\op}
\le
\|(A+tI)^{-1}\|_{\op}
\|A-B\|_{\op}
\|(B+tI)^{-1}\|_{\op}.
\]
Since \(A\succeq aI\), every eigenvalue of \(A+tI\) is at least \(a+t\). Hence
\[
\|(A+tI)^{-1}\|_{\op}
=
\frac{1}{\lambda_{\min}(A+tI)}
\le
\frac{1}{a+t}.
\]
Similarly,
\[
\|(B+tI)^{-1}\|_{\op}
\le
\frac{1}{b+t}.
\]
Therefore
\[
\|A^{1/2}-B^{1/2}\|_{\op}
\le
\frac{\|A-B\|_{\op}}{\pi}
\int_0^\infty
\frac{t^{1/2}}{(a+t)(b+t)}
\,dt .
\]

It remains to compute the scalar integral. With the substitution \(t=u^2\),
\[
\frac{1}{\pi}
\int_0^\infty
\frac{t^{1/2}}{(a+t)(b+t)}
\,dt
=
\frac{2}{\pi}
\int_0^\infty
\frac{u^2}{(a+u^2)(b+u^2)}
\,du .
\]
If \(a\ne b\), then
\[
\frac{u^2}{(a+u^2)(b+u^2)}
=
\frac{1}{a-b}
\left(
\frac{a}{a+u^2}
-
\frac{b}{b+u^2}
\right).
\]
Using
\[
\int_0^\infty \frac{c}{c+u^2}\,du
=
\frac{\pi\sqrt c}{2},
\qquad c>0,
\]
we obtain
\[
\frac{2}{\pi}
\int_0^\infty
\frac{u^2}{(a+u^2)(b+u^2)}
\,du
=
\frac{\sqrt a-\sqrt b}{a-b}
=
\frac{1}{\sqrt a+\sqrt b}.
\]
The case \(a=b\) follows by continuity, or directly from the same integral formula.
Thus
\[
\frac{1}{\pi}
\int_0^\infty
\frac{t^{1/2}}{(a+t)(b+t)}
\,dt
=
\frac{1}{\sqrt a+\sqrt b}.
\]
Substituting this into the previous estimate gives
\[
\|A^{1/2}-B^{1/2}\|_{\op}
\le
\frac{\|A-B\|_{\op}}{\sqrt a+\sqrt b},
\]
as claimed.
\end{proof}

We can now combine empirical Fisher concentration, square-root perturbation, and
estimation stability to obtain a high-probability stability bound for Fisher width
computed from the empirical Fisher matrix.

\begin{corollary}[Empirical Fisher stability]
\label{cor:empirical-fisher-stability}
Assume
\(
\|s_{\theta_0}(z)\|_2\le B
\)
almost surely, and $G\succeq \mu I$ for some \(\mu>0\). Define
\[
\eta_n(\delta)
=
B\sqrt{
\frac{2\|G\|_{\op}\log(2d/\delta)}{n}
}
+
\frac{4B^2}{3}
\frac{\log(2d/\delta)}{n}.
\]
For $n$ large enough so that
\(
\eta_n(\delta)<\mu,
\)
then with probability at least \(1-\delta\),
\[
\bigl|
w_{\widehat G_n}(T)-w_G(T)
\bigr|
\le
\frac{\eta_n(\delta)}{\sqrt{\mu}}\,w(T).
\]
\end{corollary}

\begin{proof}
Applying Weyl's inequality, we have, for every \(k=1,\ldots,d\),
\[
\lambda_k(G)-\lambda_k(\widehat G_n)
\le
\|G-\widehat G_n\|_{\op}.
\]
By Lemma~\ref{lem:empirical-fisher-concentration}, with probability at least
\(1-\delta\),
\[
\|G-\widehat G_n\|_{\op}
\le
\eta_n(\delta).
\]
Hence, on this event,
\[
\lambda_k(G)-\lambda_k(\widehat G_n)
\le
\eta_n(\delta),
\qquad k=1,\ldots,d.
\]
Thus
\[
\lambda_k(\widehat G_n)
\ge
\lambda_k(G)-\eta_n(\delta).
\]
Since \(G\succeq \mu I\), we have \(\lambda_k(G)\ge \mu\) for every \(k\). Therefore
\[
\lambda_k(\widehat G_n)
\ge
\mu-\eta_n(\delta),
\qquad k=1,\ldots,d.
\]
It follows that
\[
\widehat G_n
\succeq
\bigl(\mu-\eta_n(\delta)\bigr)I.
\]
Since \(\eta_n(\delta)<\mu\), the empirical Fisher matrix is positive definite.

Applying Lemma~\ref{lem:sqrt-perturbation-op} with
\[
A=G,
\qquad
B=\widehat G_n,
\qquad
a=\mu,
\qquad
b=\mu-\eta_n(\delta),
\]
we obtain
\[
\|\widehat G_n^{1/2}-G^{1/2}\|_{\op}
\le
\frac{\|\widehat G_n-G\|_{\op}}
{\sqrt{\mu}+\sqrt{\mu-\eta_n(\delta)}}
\le
\frac{\eta_n(\delta)}{\sqrt{\mu}}.
\]
Finally, Theorem~\ref{thm:estimation-stability} yields
\[
\bigl|
w_{\widehat G_n}(T)-w_G(T)
\bigr|
\le
\|\widehat G_n^{1/2}-G^{1/2}\|_{\op}
w(T)
\le
\frac{\eta_n(\delta)}{\sqrt{\mu}}\,w(T).
\]
\end{proof}

\begin{remark}[Structured Fisher approximations]
Theorem~\ref{thm:estimation-stability} applies equally to diagonal, block-diagonal, low-rank, and Kronecker-factored Fisher approximations. If \(\widetilde G\succeq0\) is any approximation of \(G(\theta_0)\), then
\[
|w_{\widetilde G}(T)-w_{G(\theta_0)}(T)|
\le
\|\widetilde G^{1/2}-G(\theta_0)^{1/2}\|_{\op}\,w(T).
\]
Thus the quality of the Fisher-width approximation is controlled directly by the operator-norm error in the square-root metric.
\end{remark}

\section{Curvature Corrections to Gaussian Width}
\label{sec:curvature}
Theorem~\ref{thm:estimation-stability} quantifies sensitivity to errors  in the metric representation $G(\theta_0)^{1/2}$, such as those arising  from empirical or structured Fisher approximations. It is an estimation 
result: it assumes the base point $\theta_0$ is fixed and measures how width changes as the metric changes. A complementary question is how Fisher width varies as $\theta_0$ itself moves across the statistical manifold - this is governed by curvature, and is the subject of 
this section.

The definition
\[
w_{G(\theta_0)}(T)
=
w\!\left(G(\theta_0)^{1/2}T\right)
\]
shows that Fisher width is a metric deformation of Gaussian width. Rather than measuring the Euclidean size of \(T\), it measures the size of \(T\) after the directions have been rescaled by the local Fisher information.

This is the first point where the Riemannian nature of the Fisher metric enters explicitly. Up to this point, Fisher width only uses the positive semidefinite matrix \(G(\theta_0)\) at a fixed parameter value. In contrast, the following argument uses how the metric tensor \(G(\theta)\) varies in a neighborhood of \(\theta_0\). Normal coordinates eliminate the zeroth- and first-order Euclidean
effects, so the leading nontrivial correction is second order and is controlled by curvature.

Throughout this section, we regard
\[
T\subset T_{\theta_0}\mathcal M
\]
as a compact set of tangent directions, represented in normal coordinates at \(\theta_0\). In these coordinates,
\[
G(\theta_0)=I,
\qquad
\nabla G(\theta_0)=0.
\]

\subsection{Local Metric Expansion}

Assume that, in a normal-coordinate neighborhood of \(\theta_0\), the sectional curvatures of the Fisher manifold are uniformly bounded in absolute value. That is, there exists a constant \(\kappa>0\) such that
\[
|K_\theta(\Pi)|\le \kappa
\]
for every point \(\theta\) in this neighborhood and every two-dimensional tangent plane \(\Pi\subset T_\theta\mathcal M\). In normal coordinates at \(\theta_0\), write
\[
G_\theta := G(\theta_0+\Delta\theta).
\]
Since
\[
G(\theta_0)=I,
\qquad
\nabla G(\theta_0)=0,
\]
the standard normal-coordinate expansion of the Riemannian metric \citep{doCarmo1992} gives
\[
G_\theta = I+\Delta G,
\qquad
\Delta G:=G_\theta-I, \text{ with }
\|\Delta G\|_{\op}
=
O\!\left(
\kappa \|\Delta\theta\|^2
\right).
\]
Thus the first nontrivial departure from Euclidean geometry appears only at second order in the local radius.

Expanding the matrix square root around the identity gives
\[
G_\theta^{1/2}
=
(I+\Delta G)^{1/2}
=
I+\frac12\Delta G+R,
\]
where, for \(\|\Delta G\|_{\op}\) sufficiently small,
\[
\|R\|_{\op}
\le
C\|\Delta G\|_{\op}^2.
\]

The next result quantifies the resulting deviation between Fisher width and Gaussian width.

\begin{proposition}[Curvature perturbation bound]
\label{prop:curvature-perturbation}
Let \(T\subset T_{\theta_0}\mathcal M\) be compact and represented in normal coordinates at \(\theta_0\). For a nearby point
\(
\theta=\theta_0+\Delta\theta,
\)
write
\(
G_\theta=G(\theta)=I+\Delta G.
\)
Assume
\[
\|\Delta G\|_{\op}\le c_0<\frac12.
\]
Then there exists a constant \(C>0\), depending only on \(c_0\), such that
\[
\bigl|
w_{G_\theta}(T)-w(T)
\bigr|
\le
\left(
\frac12\|\Delta G\|_{\op}
+
C\|\Delta G\|_{\op}^2
\right)w(T).
\]
In particular, if the absolute sectional curvature is bounded by \(\kappa\) in
the normal-coordinate neighborhood, then
\[
\bigl|
w_{G_\theta}(T)-w(T)
\bigr|
=
O\!\left(
\kappa\|\Delta\theta\|^2\,w(T)
\right).
\]
\end{proposition}

\begin{proof}
Since \(\|\Delta G\|_{\op}\le c_0<1/2\), the matrix square-root map admits the
first-order expansion around the identity
\[
(I+\Delta G)^{1/2}
=
I+\frac12\Delta G+R,
\qquad
\|R\|_{\op}\le C\|\Delta G\|_{\op}^2.
\]
Consequently,
\[
\|(I+\Delta G)^{1/2}-I\|_{\op}
\le
\frac12\|\Delta G\|_{\op}
+
C\|\Delta G\|_{\op}^2.
\]

Applying Theorem~\ref{thm:estimation-stability} with
\[
G_1=G_\theta=I+\Delta G,
\qquad
G_2=I,
\]
we obtain
\[
\bigl|
w_{G_\theta}(T)-w(T)
\bigr|
\le
\left(
\frac12\|\Delta G\|_{\op}
+
C\|\Delta G\|_{\op}^2
\right)w(T).
\]

Using the normal-coordinate curvature estimate stated above,
\[
\|\Delta G\|_{\op}
=
O\!\left(
\kappa\|\Delta\theta\|^2
\right),
\]
we conclude that
\[
\bigl|
w_{G_\theta}(T)-w(T)
\bigr|
=
O\!\left(
\kappa\|\Delta\theta\|^2\,w(T)
\right).
\]
\end{proof}

Proposition~\ref{prop:curvature-perturbation} shows that Fisher width and Gaussian width coincide to first order. The difference between the two quantities is controlled by the local curvature of the Fisher metric.

\subsection{Statistical model geometries}

The preceding analysis explains Fisher width as a local metric deformation of Gaussian width. We now illustrate this deformation in concrete statistical models, where the Fisher metric \(G(\theta)\) reflects noise, confidence, and sufficient-statistic covariance. Throughout, the base
point \(\theta_0\) is fixed and \(T=\rho B_2\) denotes a Euclidean ball of radius \(\rho\).

\begin{example}[Bernoulli family]
\label{ex:bernoulli}

Let
\[
p_\theta(x)
=
\theta^x(1-\theta)^{1-x},
\qquad
x\in\{0,1\},
\]
with parameter \(\theta\in(0,1)\). The Fisher information is
\[
G(\theta)
=
\frac{1}{\theta(1-\theta)}.
\]

In one dimension, \(T=\rho B_2^1=[-\rho,\rho]\). Hence
\[
w_{G(\theta_0)}(T)
=
w\!\left(
\frac{1}{\sqrt{\theta_0(1-\theta_0)}}[-\rho,\rho]
\right).
\]
Therefore
\[
w_{G(\theta_0)}(T)
=
\frac{\rho}{\sqrt{\theta_0(1-\theta_0)}}\,\mathbb E|g|
=
\rho\sqrt{\frac{2}{\pi\,\theta_0(1-\theta_0)}},
\]
where \(g\sim N(0,1)\).

Fisher width is minimized at
\[
\theta_0=\frac12,
\]
where the Bernoulli distribution has maximal entropy and
\[
w_{G(1/2)}(T)
=
\rho\sqrt{\frac8\pi}.
\]
By contrast, as \(\theta_0\to0^+\) or \(\theta_0\to1^-\), the Fisher information diverges and
\[
w_{G(\theta_0)}(T)\to\infty.
\]

Thus Fisher width detects the boundary singularity of the Bernoulli family: although the Euclidean perturbation set \(T=[-\rho,\rho]\) is fixed, its statistical size diverges as the base distribution approaches a deterministic
Bernoulli law.
\end{example}

\begin{example}[Gaussian location-scale family]
\label{ex:gaussian-location-scale}

Let
\[
X\sim N(\mu,\tau),
\qquad
\mu\in\mathbb R,
\qquad
\tau=\sigma^2>0,
\]
with parameter
\[
\theta=(\mu,\tau)\in\mathbb R\times(0,\infty).
\]
The log-likelihood is
\[
\ell(\mu,\tau;x)
=
-\frac12\log(2\pi\tau)
-
\frac{(x-\mu)^2}{2\tau}.
\]
A direct computation gives
\[
G(\mu,\tau)
=
\begin{pmatrix}
\dfrac{1}{\tau} & 0\\
0 & \dfrac{1}{2\tau^2}
\end{pmatrix}.
\]
Thus the location and variance directions have different statistical scales:
for \(v=(v_\mu,v_\tau)\),
\[
\|v\|_{G(\mu,\tau)}^2
=
\frac{1}{\tau}v_\mu^2
+
\frac{1}{2\tau^2}v_\tau^2.
\]

Now take
\[
T=\rho B_2^2\subset\mathbb R^2.
\]
Since
\[
G(\mu,\tau)^{1/2}
=
\begin{pmatrix}
\tau^{-1/2} & 0\\
0 & (\sqrt{2}\tau)^{-1}
\end{pmatrix},
\]
we obtain
\[
w_{G(\mu,\tau)}(T)
=
\rho\,\mathbb E
\|G(\mu,\tau)^{1/2}g\|_2,
\qquad
g\sim N(0,I_2).
\]
Equivalently,
\[
w_{G(\mu,\tau)}(T)
=
\rho\,
\mathbb E
\left[
\left(
\frac{1}{\tau}g_1^2
+
\frac{1}{2\tau^2}g_2^2
\right)^{1/2}
\right],
\]
where \(g=(g_1,g_2)\sim N(0,I_2)\).

The spectral sandwich bound gives
\[
\sqrt{
\min\left\{
\frac{1}{\tau},
\frac{1}{2\tau^2}
\right\}
}
\,w(T)
\le
w_{G(\mu,\tau)}(T)
\le
\sqrt{
\max\left\{
\frac{1}{\tau},
\frac{1}{2\tau^2}
\right\}
}
\,w(T).
\]

Thus Fisher width is not merely the Euclidean size of the perturbation set.
Although \(T=\rho B_2^2\) is a round Euclidean ball in the \((\mu,\tau)\)-space,
the Fisher metric transforms it into an anisotropic ellipsoid. The mean
direction is rescaled by \(\tau^{-1/2}\), while the variance direction is
rescaled by \((\sqrt{2}\tau)^{-1}\). These scales vary with the noise level
\(\tau\).
\end{example}

\begin{example}[Logistic regression]
\label{ex:logistic}

Let \(X\in\mathbb R^d\) be a covariate with distribution \(P_X\), and let
\(Y\in\{0,1\}\) satisfy the conditional Bernoulli model
\[
\mathbb P_\theta(Y=1\mid X=x)
=
\sigma(\theta^\top x),
\qquad
\sigma(t)
=
\frac{1}{1+e^{-t}}.
\]
Equivalently,
\[
\mathbb P_\theta(Y=y\mid X=x)
=
\sigma(\theta^\top x)^y
\bigl(1-\sigma(\theta^\top x)\bigr)^{1-y},
\qquad y\in\{0,1\}.
\]

The log-likelihood for one observation $(X, Y)$ is
\[
\log \mathbb{P}_\theta(Y \mid X)
= Y \theta^\top X - \log(1 + e^{\theta^\top X}),
\]
and differentiating with respect to $\theta$ gives the score
\[
\nabla_\theta \log \mathbb{P}_\theta(Y \mid X)
= \bigl(Y - \sigma(\theta^\top X)\bigr)X.
\]

Therefore the Fisher information matrix is
\[
G(\theta)
=
\mathbb E_{X\sim P_X}
\left[
\sigma(\theta^\top X)
\bigl(1-\sigma(\theta^\top X)\bigr)
XX^\top
\right].
\]

The scalar factor
\[
\sigma(\theta^\top X)
\bigl(1-\sigma(\theta^\top X)\bigr)
\]
is the conditional variance of \(Y\mid X\). It is maximized at
\[
\theta^\top X=0,
\]
where it equals \(1/4\), and decreases to zero as
\[
|\theta^\top X|\to\infty.
\]

Suppose that, along a sequence of parameters \(\theta_m\),
\[
|\theta_m^\top X|\to\infty
\qquad
P_X\text{-almost surely},
\]
and assume
\[
\mathbb E\|X\|_2^2<\infty.
\]
Then
\[
\sigma(\theta_m^\top X)
\bigl(1-\sigma(\theta_m^\top X)\bigr)
XX^\top
\to 0
\qquad
P_X\text{-almost surely}.
\]
Moreover,
\[
0
\preceq
\sigma(\theta_m^\top X)
\bigl(1-\sigma(\theta_m^\top X)\bigr)
XX^\top
\preceq
\frac14 XX^\top.
\]
By dominated convergence,
\[
G(\theta_m)\to 0.
\]
In particular,
\[
\lambda_{\max}(G(\theta_m))\to0.
\]

The spectral sandwich bound gives
\[
0
\le
w_{G(\theta_m)}(T)
\le
\sqrt{\lambda_{\max}(G(\theta_m))}\,w(T)
\longrightarrow
0.
\]

Thus Fisher width vanishes as as $\|\theta_m^\top X\| \to \infty$. In this regime, the model assigns probabilities close to \(0\) or \(1\), the Bernoulli variance collapses, and the Fisher metric contracts.

This example suggests a connection between Fisher width and model confidence. The experiments in Section~\ref{sec:experiments} examine related Fisher-width
behavior in trained logistic, softmax, and ridge models.
\end{example}
\paragraph{Summary.}

These examples illustrate three qualitatively distinct Fisher geometries. The Bernoulli family exhibits a singular geometry near the boundary of the parameter space, where Fisher width diverges. The Gaussian location-scale family shows that even a simple Euclidean perturbation set can be distorted anisotropically: mean and variance directions are measured on different statistical scales, and these scales vary with the noise level. Logistic regression exhibits a data-dependent geometry whose local Fisher structure contracts as predictions become more confident.

Together, these examples show that Fisher width captures geometric structure that classical Gaussian width treats uniformly: boundary singularities in Bernoulli models, anisotropic noise-dependent scaling in Gaussian location-scale models, and confidence-driven contraction in logistic
regression.

\section{Generalization Bound}
\label{sec:generalization}
\subsection{Setup and assumptions}
\label{subsec:gen_setup}

Let \(T\subset\mathbb R^d\) be a compact parameter set, and let
\(\ell(\theta;z)\) be a loss function indexed by \(\theta\in T\), where
\(z=(x,y)\). We assume that
\[
0\in T,
\qquad
\ell(0;z)=0
\]
for all \(z\). Define the population and empirical risks by
\[
R(\theta)
=
\mathbb E_{z\sim\mathcal D}[\ell(\theta;z)],
\qquad
\widehat R_n(\theta)
=
\frac1n\sum_{i=1}^n \ell(\theta;z_i),
\]
where \(z_1,\ldots,z_n\) are independent samples from \(\mathcal D\).

\begin{remark}[Centering convention]
\label{rem:centering}
The assumptions \(0\in T\) and \(\ell(0;z)=0\) are a centering convention for the loss class. They ensure that the complexity term depends only on the increment geometry of the class. Without such a convention, a singleton class
could have \(T-T=\{0\}\) and hence zero width, while its loss may still depend on \(z\) and exhibit a nonzero generalization gap. Equivalently, in the general
case one may replace \(\ell(\theta;z)\) by \(\ell(\theta;z)-\ell(\theta_{\rm ref};z)\) for a fixed reference point \(\theta_{\rm ref}\in T\).
\end{remark}

\begin{assumption}[Boundedness]
\label{ass:bounded}
There exists \(B>0\) such that
\(
0 \le \ell(\theta;z) \le B
\)
for all \(\theta\in T\) and all samples \(z\).
\end{assumption}

\begin{assumption}[Base-point Fisher-Lipschitzness]
\label{ass:fisher-lipschitz}
Let \(\theta_0\in\Theta\) be fixed and write
\[
G_0:=G(\theta_0).
\]
There exists \(L>0\) such that, for every sample \(z\) and every
\(\theta_1,\theta_2\in T\),
\[
|\ell(\theta_1;z)-\ell(\theta_2;z)|
\le
L\|\theta_1-\theta_2\|_{G_0},
\]
where
\[
\|h\|_{G_0}
=
\sqrt{h^\top G_0h}.
\]
\end{assumption}

\begin{remark}[Interpretation]
Base-point Fisher-Lipschitzness is not a global Lipschitz condition on the statistical manifold. It is a local regularity condition measured in the tangent geometry induced by the Fisher metric at the reference point \(\theta_0\). When the Fisher metric varies slowly over \(T\), this provides a natural analogue of Euclidean Lipschitzness adapted to statistical distinguishability.
\end{remark}

The base-point Fisher-Lipschitz condition in
Assumption~\ref{ass:fisher-lipschitz} is stated using the fixed metric
\(G_0=G(\theta_0)\). In applications, however, gradient bounds are often more
naturally expressed in the local Fisher metric \(G(\theta)\), which varies with
the parameter. The following proposition shows that such a local bound implies
base-point Fisher-Lipschitzness whenever the metric does not vary too much over
the parameter set \(T\).

\begin{proposition}[Metric comparability implies base-point Fisher-Lipschitzness]
\label{prop:metric-comparability-lipschitz}
Let \(T\subset\mathbb R^d\) be convex and let \(G_0=G(\theta_0)\succ0\).
Suppose that, for some \(\varepsilon\in(0,1)\),
\[
G(\theta)
\preceq
(1+\varepsilon)G_0
\]
for all \(\theta\in T\). Assume also that, for every sample \(z\) and every
\(\theta\in T\),
\[
\|\nabla_\theta\ell(\theta;z)\|_{G(\theta)^{-1}}
\le L.
\]
Then, for every sample \(z\) and every \(\theta_1,\theta_2\in T\),
\[
|\ell(\theta_1;z)-\ell(\theta_2;z)|
\le
L\sqrt{1+\varepsilon}
\|\theta_1-\theta_2\|_{G_0}.
\]
In particular, Assumption~\ref{ass:fisher-lipschitz} holds with constant
\(L\sqrt{1+\varepsilon}\).
\end{proposition}

\begin{proof}
Fix \(z\) and let \(\theta_1,\theta_2\in T\). Define
\[
\gamma_t=(1-t)\theta_1+t\theta_2,
\qquad t\in[0,1].
\]
Since \(T\) is convex, \(\gamma_t\in T\) for all \(t\in[0,1]\). By the
fundamental theorem of calculus,
\[
\ell(\theta_2;z)-\ell(\theta_1;z)
=
\int_0^1
\left\langle
\nabla_\theta\ell(\gamma_t;z),
\theta_2-\theta_1
\right\rangle
\,dt.
\]
Applying Cauchy--Schwarz in the Fisher metric \(G(\gamma_t)\), we obtain
\[
|\ell(\theta_2;z)-\ell(\theta_1;z)|
\le
\int_0^1
\|\nabla_\theta\ell(\gamma_t;z)\|_{G(\gamma_t)^{-1}}
\|\theta_2-\theta_1\|_{G(\gamma_t)}
\,dt.
\]
By assumption,
\[
\|\nabla_\theta\ell(\gamma_t;z)\|_{G(\gamma_t)^{-1}}
\le L.
\]
Moreover, since
\[
G(\gamma_t)\preceq(1+\varepsilon)G_0,
\]
we have
\[
\|\theta_2-\theta_1\|_{G(\gamma_t)}
\le
\sqrt{1+\varepsilon}
\|\theta_2-\theta_1\|_{G_0}.
\]
Therefore
\[
|\ell(\theta_2;z)-\ell(\theta_1;z)|
\le
L\sqrt{1+\varepsilon}
\|\theta_2-\theta_1\|_{G_0}.
\]
\end{proof}

The Fisher-Lipschitz assumption is model-dependent, since it requires a uniform bound on the dual Fisher norm of the loss gradients. For the classification models used in the experiments, this condition follows from standard Fisher-leverage bounds. In particular, logistic regression and
multiclass softmax regression satisfy Assumption~\ref{ass:fisher-lipschitz}
under bounded-feature and Fisher nondegeneracy conditions; the precise statements and proofs are given in Appendix~\ref{app:fisher-lipschitz}.

\subsection{Main theorem}
\label{subsec:gen_theorem}
\begin{theorem}[Local Fisher-width generalization bound]
\label{thm:generalization}
Let
\[
G_0:=G(\theta_0).
\]
Under Assumptions~\ref{ass:bounded} and~\ref{ass:fisher-lipschitz}, for every
\(\delta\in(0,1)\), with probability at least \(1-\delta\),
\[
\sup_{\theta\in T}
\bigl|
R(\theta)-\widehat R_n(\theta)
\bigr|
\le
\frac{
C L\, w_{G_0}(T-T)
}{
\sqrt n
}
+
B
\sqrt{
\frac{\log(1/\delta)}{2n}
}.
\]
\end{theorem}

\begin{proof}
The proof proceeds in four steps.

\medskip
\noindent\textbf{Step 1: Symmetrization.}
By the standard symmetrization lemma,
\[
\mathbb E
\sup_{\theta\in T}
\bigl|
R(\theta)-\widehat R_n(\theta)
\bigr|
\le
\frac{2}{n}
\mathbb E_{\boldsymbol\sigma,\mathbf z}
\sup_{\theta\in T}
\left|
\sum_{i=1}^n
\sigma_i\ell(\theta;z_i)
\right|,
\]
where
\[
\sigma_i\overset{\mathrm{i.i.d.}}{\sim}\operatorname{Unif}\{-1,+1\}
\]
are independent Rademacher variables and \(z_i=(x_i,y_i)\).

For fixed data, define the Rademacher process
\[
X_\theta
=
\sum_{i=1}^n
\sigma_i\ell(\theta;z_i),
\qquad
\theta\in T.
\]
We first control
\[
\mathbb E_{\boldsymbol\sigma}
\sup_{\theta\in T}X_\theta.
\]

\medskip
\noindent\textbf{Step 2: Subgaussian increments in the Fisher metric.}
For fixed data and for \(\theta_1,\theta_2\in T\),
\[
X_{\theta_1}-X_{\theta_2}
=
\sum_{i=1}^n
\sigma_i
\bigl(
\ell(\theta_1;z_i)-\ell(\theta_2;z_i)
\bigr).
\]
By the standard subgaussian bound for Rademacher sums,
\[
\|X_{\theta_1}-X_{\theta_2}\|_{\psi_2}
\le
C
\left(
\sum_{i=1}^n
|\ell(\theta_1;z_i)-\ell(\theta_2;z_i)|^2
\right)^{1/2}.
\]
Using Assumption~\ref{ass:fisher-lipschitz}, we obtain
\[
\|X_{\theta_1}-X_{\theta_2}\|_{\psi_2}
\le
CL\sqrt n\,
\|\theta_1-\theta_2\|_{G_0}.
\]
Thus \((X_\theta)_{\theta\in T}\) is a subgaussian process with respect to the
metric
\[
d(\theta_1,\theta_2)
=
L\sqrt n\,
\|\theta_1-\theta_2\|_{G_0}.
\]

\medskip
\noindent\textbf{Step 3: Chaining and Fisher width.}
By the generic chaining bound for subgaussian processes,
\[
\mathbb E_{\boldsymbol\sigma}
\sup_{\theta\in T}X_\theta
\le
CL\sqrt n\,
\gamma_2(T,\|\cdot\|_{G_0}).
\]
Since
\[
\|\theta_1-\theta_2\|_{G_0}
=
\|G_0^{1/2}\theta_1-G_0^{1/2}\theta_2\|_2,
\]
we identify \(T\) with the Fisher-rescaled set
\[
\widetilde T:=G_0^{1/2}T.
\]
Hence
\[
\gamma_2(T,\|\cdot\|_{G_0})
=
\gamma_2(\widetilde T,\|\cdot\|_2).
\]
By Talagrand's majorizing measure theorem \citep{talagrand2014},
\[
\gamma_2(\widetilde T,\|\cdot\|_2)
\le
C\,w(\widetilde T).
\]
Since \(0\in T\), we have
\[
T\subset T-T,
\]
and therefore
\[
\widetilde T
=
G_0^{1/2}T
\subset
G_0^{1/2}(T-T).
\]
Thus
\[
w(\widetilde T)
\le
w\bigl(G_0^{1/2}(T-T)\bigr)
=
w_{G_0}(T-T).
\]
Consequently,
\[
\mathbb E_{\boldsymbol\sigma}
\sup_{\theta\in T}X_\theta
\le
CL\sqrt n\,
w_{G_0}(T-T).
\]

We now apply the same argument to the process
\((-X_\theta)_{\theta\in T}\). For every
\(\theta_1,\theta_2\in T\),
\[
(-X_{\theta_1})-(-X_{\theta_2})
=
-(X_{\theta_1}-X_{\theta_2}),
\]
and hence
\[
|(-X_{\theta_1})-(-X_{\theta_2})|
=
|X_{\theta_1}-X_{\theta_2}|.
\]
Thus \((-X_\theta)_{\theta\in T}\) has the same increment size as
\((X_\theta)_{\theta\in T}\), so the same chaining bound gives
\[
\mathbb E_{\boldsymbol\sigma}
\sup_{\theta\in T}(-X_\theta)
\le
CL\sqrt n\,
w_{G_0}(T-T).
\]
Since
\[
\sup_{\theta\in T}|X_\theta|
\le
\sup_{\theta\in T}X_\theta
+
\sup_{\theta\in T}(-X_\theta),
\]
we conclude that
\[
\mathbb E_{\boldsymbol\sigma}
\sup_{\theta\in T}|X_\theta|
\le
CL\sqrt n\,
w_{G_0}(T-T).
\]
Combining this estimate with Step 1 yields
\[
\mathbb E
\sup_{\theta\in T}
\bigl|
R(\theta)-\widehat R_n(\theta)
\bigr|
\le
\frac{
CL\,w_{G_0}(T-T)
}{
\sqrt n
}.
\]

\medskip
\noindent\textbf{Step 4: Concentration.}
Define
\[
f(z_1,\ldots,z_n)
=
\sup_{\theta\in T}
\bigl|
R(\theta)-\widehat R_n(\theta)
\bigr|.
\]
Changing one sample \(z_i\) changes \(\widehat R_n(\theta)\) by at most
\(B/n\), uniformly over \(\theta\), because \(\ell(\theta;z)\in[0,B]\).
Hence \(f\) satisfies the bounded differences condition with constants
\[
c_i=\frac{B}{n}.
\]
By McDiarmid's inequality,
\[
\mathbb P
\left(
f-\mathbb Ef\ge t
\right)
\le
\exp\left(
-\frac{2nt^2}{B^2}
\right).
\]
Taking
\[
t=
B\sqrt{\frac{\log(1/\delta)}{2n}}
\]
and adding this concentration term to the expectation bound proves the claim.
\end{proof}

\begin{remark}
The appearance of the difference set
\[
T-T
=
\{u-v:\;u,v\in T\}
\]
is standard in empirical-process theory and avoids
symmetry assumptions on the parameter class. If $T$ is convex, symmetric, and contains the origin,
then
\[
T-T \subset 2T,
\]
and therefore
\[
w_G(T-T;\theta_0)
\le
2\,w_G(T;\theta_0).
\]
\end{remark}

\subsection{Tightness for Exponential Families}
\label{subsec:tightness}

Theorem~\ref{thm:generalization} gives an upper bound controlled by Fisher
width. We now show that this scale is sharp in a canonical setting:
exponential families with negative log-likelihood loss. The result should be
understood as a sharpness example rather than a matching lower bound for all
Fisher-Lipschitz losses.

Let
\[
  p_\theta(z)
  = h(z)\exp\bigl(\theta^\top\phi(z)-A(\theta)\bigr),
  \qquad \theta\in\Theta\subseteq\mathbb{R}^d,
\]
be a regular exponential family, and fix a base point $\theta_0\in\Theta$.
Write $G_0 = G(\theta_0) = \operatorname{Cov}_{p_{\theta_0}}(\phi(Z))$ and
assume $G_0\succ 0$. For $u\in\mathbb{R}^d$, define the local perturbation
loss
\[
  \widetilde\ell(u;z)
  = -\log p_{\theta_0+u}(z) + \log p_{\theta_0}(z),
\]
and assume $\theta_0+u\in\Theta$ for all $u\in T$. We take $T = rB_2^d$.

\begin{lemma}[Exact gap for exponential-family NLL]
  \label{lem:exact-gap-expfam}
  Let $Z_1,\ldots,Z_n\overset{\mathrm{i.i.d.}}{\sim}p_{\theta_0}$, and define
  \[
    \overline\phi_n = \frac{1}{n}\sum_{i=1}^n\phi(Z_i),
    \qquad
    \Delta_n = \overline\phi_n - \mathbb{E}_{p_{\theta_0}}\phi(Z).
  \]
  Then
  \[
    \sup_{u\in rB_2^d}
    \bigl|R(u) - \widehat{R}_n(u)\bigr|
    = r\|\Delta_n\|_2,
  \]
  where $R$ and $\widehat{R}_n$ are the population and empirical risks
  associated with $\widetilde\ell(u;z)$.
  The proof is given in Appendix~\ref{app:tightness}.
\end{lemma}

\begin{theorem}[Sharp Fisher-width scale for exponential families]
  \label{thm:tightness-expfam}
  Assume $\mathbb{E}_{p_{\theta_0}}\|\phi(Z)\|_2^2 < \infty$. Then
  \[
    \lim_{n\to\infty}
    \sqrt{n}\;
    \mathbb{E}\!\left[
      \sup_{u\in rB_2^d}
      \bigl|R(u) - \widehat{R}_n(u)\bigr|
    \right]
    = w_{G_0}(rB_2^d).
  \]
  Consequently, for every $\varepsilon\in(0,1)$ there exists $N_\varepsilon$
  such that for all $n\ge N_\varepsilon$,
  \[
    \mathbb{E}\!\left[
      \sup_{u\in rB_2^d}
      \bigl|R(u) - \widehat{R}_n(u)\bigr|
    \right]
    \ge \frac{(1-\varepsilon)\,w_{G_0}(rB_2^d)}{\sqrt{n}}.
  \]
  The proof is given in Appendix~\ref{app:tightness}.
\end{theorem}

\begin{remark}[Interpretation]
  \label{rem:tightness-interp}
  Theorem~\ref{thm:tightness-expfam} shows that Fisher width gives the sharp
  asymptotic scale of the uniform generalization gap for local perturbations
  in exponential families with NLL loss. The empirical fluctuation of the NLL
  is linear in the perturbation $u$ --- because the deterministic term
  $A(\theta_0+u)-A(\theta_0)$ cancels in $R(u)-\widehat{R}_n(u)$ --- and the
  covariance of this fluctuation is exactly $G_0$. Thus the same Gaussian
  process that defines Fisher width also governs the asymptotic size of the
  generalization gap.
\end{remark}

\begin{remark}[Relation to the upper bound]
  \label{rem:tightness-vs-ub}
  For the symmetric Euclidean ball $T = rB_2^d$,
  \[
    T - T = 2T, \qquad w_{G_0}(T-T) = 2\,w_{G_0}(T).
  \]
  Thus Theorem~\ref{thm:tightness-expfam} matches the Fisher-width dependence
  in Theorem~\ref{thm:generalization} up to universal constants and the
  standard difference-set factor. It shows that the Fisher-width term in the
  upper bound is not an artifact of the proof method.
\end{remark}

\begin{remark}[Scope of the lower bound]
  \label{rem:lb-scope}
  The lower bound does not extend to arbitrary Fisher-Lipschitz losses without
  additional assumptions. Fisher-Lipschitzness controls how rapidly the loss
  changes with the parameter, but does not ensure nontrivial data-dependent
  fluctuations. For example, a loss that is constant on $T$ is
  Fisher-Lipschitz yet has zero generalization gap regardless of $w_G(T)$.
  Matching lower bounds therefore require a non-degeneracy condition on top of
  Fisher-Lipschitzness. Exponential-family NLL satisfies this condition
  exactly, because the covariance of the score is precisely $G_0$.
\end{remark}

\subsection{Fisher Width of Euclidean Balls}
\label{subsec:fisher_euclidean_ball}

Theorem~\ref{thm:generalization} shows that local generalization is governed by Fisher width. We next evaluate this quantity for a canonical parameter
set, the Euclidean ball
\(
T=rB_2^d.
\)
As observed in Example~2.4,
\[
w_G(rB_2^d)
=
r\,\mathbb E\|G^{1/2}g\|_2
\le
r\sqrt{\operatorname{Tr}(G)}.
\]
The following result shows that this upper bound is sharp up to universal constants. Thus, for Euclidean balls, Fisher width is controlled precisely by the total Fisher information \(\operatorname{Tr}(G)\), giving an intrinsic
effective dimension for the local model.

\begin{theorem}[Fisher width of Euclidean balls]
\label{thm:fisher_ball_width}
Let \(G\succeq0\) and let
\(
T=rB_2^d.
\)
Then
\[
\sqrt{\frac{2}{\pi}}\,
r\sqrt{\operatorname{Tr}(G)}
\le
w_G(T)
\le
r\sqrt{\operatorname{Tr}(G)}.
\]
Consequently,
\[
w_G(rB_2^d)
\asymp
r\sqrt{\operatorname{Tr}(G)},
\]
where the implied constants are universal.
\end{theorem}

\begin{proof}
The upper bound follows from Example~2.4, so it remains only to prove the
lower bound. Let \(\lambda_1,\ldots,\lambda_d\) be the eigenvalues of \(G\).
By rotational invariance of the standard Gaussian,
\[
w_G(rB_2^d)
=
r\,\mathbb E\Bigl(\sum_{i=1}^d \lambda_i g_i^2\Bigr)^{1/2},
\qquad
g_i\overset{\mathrm{i.i.d.}}{\sim}N(0,1).
\]
If \(\Tr(G)=0\), then \(G=0\) and the claim is trivial. Otherwise, set
\[
p_i=\frac{\lambda_i}{\Tr(G)}.
\]
Then \(p=(p_1,\ldots,p_d)\) lies in the probability simplex, and
\[
w_G(rB_2^d)
=
r\sqrt{\Tr(G)}
\,
\mathbb E\Bigl(\sum_{i=1}^d p_i g_i^2\Bigr)^{1/2}.
\]
It remains to lower bound the last expectation uniformly over all probability
vectors \(p\).

For each fixed realization of \(g\), the map
\[
p\mapsto
\Bigl(\sum_{i=1}^d p_i g_i^2\Bigr)^{1/2}
\]
is concave on the simplex. Since \(p=\sum_i p_i e_i\), concavity gives
\[
\Bigl(\sum_{i=1}^d p_i g_i^2\Bigr)^{1/2}
\ge
\sum_{i=1}^d p_i |g_i|.
\]
Taking expectations,
\[
\mathbb E\Bigl(\sum_{i=1}^d p_i g_i^2\Bigr)^{1/2}
\ge
\sum_{i=1}^d p_i\,\mathbb E|g_i|
=
\sqrt{\frac{2}{\pi}}.
\]
Therefore
\[
w_G(rB_2^d)
\ge
\sqrt{\frac{2}{\pi}}\,
r\sqrt{\Tr(G)}.
\]
This proves the desired lower bound.
\end{proof}

\subsection{Effective Fisher Dimension}
\label{subsec:effective-fisher-dimension}
Theorem~\ref{thm:fisher_ball_width} suggests a natural way to separate the scale of the Fisher metric from its effective dimensionality. For a nonzero matrix \(G\succeq0\), define
\[
d_F(G)
:=
\frac{\Tr(G)}{\lambda_{\max}(G)}.
\]
which we call the \emph{effective Fisher dimension}.

The normalization by \(\lambda_{\max}(G)\) removes the largest local Fisher scale. What remains is a dimension-like quantity measuring how many Fisher-relevant directions contribute to the trace. Indeed,
\[
\Tr(G)=\lambda_{\max}(G)d_F(G).
\]
Substituting this identity into Theorem~\ref{thm:fisher_ball_width} immediately yields the following characterization.

\begin{corollary}[Effective Fisher dimension]
\label{cor:effective_fisher_dimension}
Let \(G\succeq0\) be nonzero and let \(T=rB_2^d\). Then
\[
\frac{2}{\pi}\,
d_F(G)
\le
\frac{w_G(T)^2}{r^2\lambda_{\max}(G)}
\le
d_F(G).
\]
Equivalently,
\[
w_G(rB_2^d)^2
\asymp
r^2\lambda_{\max}(G)d_F(G),
\]
with universal constants.
\end{corollary}

\begin{remark}[Interpretation of \(d_F\)]
\label{rem:effective_fisher_dimension}
The quantity
\[
d_F(G)
=
\frac{\sum_{i=1}^d\lambda_i(G)}{\lambda_{\max}(G)}
=
\frac{\|G^{1/2}\|_F^2}{\|G^{1/2}\|_{\mathrm{op}}^2}
\]
is the stable rank of \(G^{1/2}\). It separates the overall scale of the
Fisher metric from its effective dimensionality: \(\lambda_{\max}(G)\)
measures the largest local Fisher sensitivity, while \(d_F(G)\) measures how
many Fisher-relevant directions contribute to the geometry.

In the isotropic case \(G=\lambda I_d\), one has
\[
d_F(G)=d.
\]
More generally, if \(\operatorname{rank}(G)=r\), then
\[
d_F(G)\le r.
\]
Thus \(d_F(G)\) depends on the Fisher spectrum rather than the ambient
parameter dimension.
\end{remark}

As a direct consequence, the local generalization bound can be expressed in
terms of the effective Fisher dimension. Indeed, if \(T\subseteq rB_2^d\), then
\[
T-T\subseteq 2rB_2^d.
\]
Hence
\[
w_{G_0}(T-T)
\le
w_{G_0}(2rB_2^d)
\lesssim
r\sqrt{\lambda_{\max}(G_0)d_F(G_0)}.
\]
Substituting this estimate into Theorem~\ref{thm:generalization} gives, with
probability at least \(1-\delta\),
\[
\sup_{\theta\in T}
|R(\theta)-\widehat R_n(\theta)|
\lesssim
\frac{
Lr\sqrt{\lambda_{\max}(G_0)d_F(G_0)}
}
{\sqrt n}
+
B\sqrt{\frac{\log(1/\delta)}{2n}}.
\]

\begin{remark}[Effective versus ambient dimension]
\label{rem:effective_vs_ambient}
In the isotropic case \(G_0=\lambda I_d\), the effective Fisher dimension is
\(d_F(G_0)=d\), and the complexity term scales as
\[
r\sqrt{\frac{\lambda d}{n}}.
\]
When the Fisher spectrum is concentrated on a lower-dimensional subspace,
\(d_F(G_0)\ll d\), and the bound replaces the ambient dimension by the
effective Fisher dimension. Thus, for Euclidean parameter sets, local
generalization is governed by the Fisher spectrum rather than by the raw
number of parameters.
\end{remark}
\section{Computational Estimation of Fisher Width}
\label{sec:computation}

Fisher width depends on the Fisher information matrix \(G=G(\theta_0)\), which is typically unavailable in closed form. In practice, one can replaces \(G\) by a computable positive semidefinite approximation \(\widehat G\succeq0\), such as an empirical, diagonal, low-rank, or structured
Fisher approximation.

The basic stability principle is
\[
\bigl|
w_{\widehat G}(T)-w_G(T)
\bigr|
\le
\|\widehat G^{1/2}-G^{1/2}\|_{\op}\,w(T),
\]
as established in Theorem~\ref{thm:estimation-stability}. Thus Fisher width can be approximated by
\[
w_{\widehat G}(T)
=
w(\widehat G^{1/2}T),
\]
with an error controlled by the square-root approximation error of the Fisher matrix.

Assume that the support function
\[
h_T(c)
=
\sup_{v\in T}\langle c,v\rangle
\]
can be evaluated efficiently. Then
\[
w_{\widehat G}(T)
=
\mathbb E_g h_T(\widehat G^{1/2}g),
\qquad
g\sim N(0,I_d),
\]
can be estimated by Monte Carlo sampling. The estimators below differ in how \(\widehat G^{1/2}\) is constructed.

\subsection{Full Empirical Fisher Estimator}
\label{subsec:full-fisher}

Given i.i.d.\ samples \(z_1,\ldots,z_n\), define the empirical Fisher matrix
\[
\widehat G_n
=
\frac1n
\sum_{i=1}^n
s_{\theta_0}(z_i)s_{\theta_0}(z_i)^\top,
\]
where
\[
s_{\theta_0}(z)
=
\nabla_\theta \log p_\theta(z)\big|_{\theta=\theta_0}
\]
is the score. The corresponding Monte Carlo estimator is
\[
\widehat w_{\widehat G_n}^{\mathrm{MC}}(T)
=
\frac1B
\sum_{b=1}^B
h_T(\widehat G_n^{1/2}g_b),
\qquad
g_b\overset{\mathrm{i.i.d.}}{\sim}N(0,I_d).
\]

The total error decomposes into Monte Carlo error and Fisher estimation error:
\[
\bigl|
\widehat w_{\widehat G_n}^{\mathrm{MC}}(T)-w_G(T)
\bigr|
\le
\underbrace{
\bigl|
\widehat w_{\widehat G_n}^{\mathrm{MC}}(T)-w_{\widehat G_n}(T)
\bigr|
}_{\mathrm{Monte\ Carlo\ error}}
+
\underbrace{
\bigl|
w_{\widehat G_n}(T)-w_G(T)
\bigr|
}_{\mathrm{Fisher\ estimation\ error}}.
\]
The Monte Carlo term vanishes as \(B\to\infty\) by the law of large numbers.
The second term is controlled by stability. In particular, on any event where
\[
\|\widehat G_n^{1/2}-G^{1/2}\|_{\op}
\le
\varepsilon_n(\delta),
\]
one has
\[
\bigl|
w_{\widehat G_n}(T)-w_G(T)
\bigr|
\le
\varepsilon_n(\delta)\,w(T).
\]
For dense matrices, the dominant costs are forming \(\widehat G_n\) and
computing a spectral factorization, which scale as
\[
O(nd^2+d^3).
\]

\subsection{Low-Rank Fisher Approximation}
\label{subsec:low-rank}

When the empirical Fisher spectrum decays rapidly, one may replace the full spectral factorization by a rank-\(k\) approximation. Let
\[
\widehat G_n
=
U\Lambda U^\top
\]
be the eigendecomposition of the empirical Fisher matrix, and let
\[
\widehat G_k
=
U_k\Lambda_kU_k^\top
\]
be its rank-\(k\) truncation. The corresponding low-rank approximation is
\[
w_{\widehat G_k}(T)
=
\mathbb E_g h_T(\widehat G_k^{1/2}g).
\]

\begin{proposition}[Low-rank approximation error]
\label{prop:low-rank-computation}
Let \(\widehat G_k\) be the rank-\(k\) truncation of \(\widehat G_n\). Then,
for every \(T\subseteq\mathbb R^d\),
\[
\bigl|
w_{\widehat G_k}(T)-w_{\widehat G_n}(T)
\bigr|
\le
\sqrt{\lambda_{k+1}(\widehat G_n)}\,w(T),
\]
with the convention that \(\lambda_{k+1}(\widehat G_n)=0\) if
\(k\ge \operatorname{rank}(\widehat G_n)\). Consequently, if
\[
\bigl|
w_{\widehat G_n}(T)-w_G(T)
\bigr|
\le
\varepsilon_n(\delta)\,w(T)
\]
with probability at least \(1-\delta\), then
\[
\bigl|
w_{\widehat G_k}(T)-w_G(T)
\bigr|
\le
\Bigl(
\sqrt{\lambda_{k+1}(\widehat G_n)}
+
\varepsilon_n(\delta)
\Bigr)w(T)
\]
with probability at least \(1-\delta\).
\end{proposition}

\begin{proof}
By the stability principle,
\[
\bigl|
w_{\widehat G_k}(T)-w_{\widehat G_n}(T)
\bigr|
\le
\|\widehat G_k^{1/2}-\widehat G_n^{1/2}\|_{\op}\,w(T).
\]
Since \(\widehat G_k\) is obtained by truncating the spectrum of
\(\widehat G_n\),
\[
\widehat G_n-\widehat G_k
=
\sum_{i>k}
\lambda_i(\widehat G_n)u_i u_i^\top.
\]
Therefore,
\[
\|\widehat G_k^{1/2}-\widehat G_n^{1/2}\|_{\op}
=
\sqrt{\lambda_{k+1}(\widehat G_n)}.
\]
This proves the first inequality. The second follows by the triangle
inequality.
\end{proof}

Thus the computational error is controlled by the discarded Fisher spectrum.
If the eigenvalues of \(\widehat G_n\) decay rapidly, a small number of leading
Fisher directions suffices. Randomized SVD or sketching methods can compute
\(\widehat G_k\), without explicitly forming the dense empirical Fisher
matrix, at substantially lower cost than a full eigendecomposition, typically
on the order of
\[
O(ndk+dk^2)
\]
when applied to score samples \citep{HalkoMartinssonTropp2011}.
\subsection{Score-Norm Estimator for Euclidean Balls}
\label{subsec:score-norm}

For Euclidean balls, Fisher width admits an even cheaper estimator. Let
\[
T=rB_2^d.
\]
The identity
\[
\Tr(G)
=
\mathbb E\|s_{\theta_0}(z)\|_2^2
\]
suggests estimating the trace by the empirical score norm
\[
\widehat \tau_n
=
\frac1n
\sum_{i=1}^n
\|s_{\theta_0}(z_i)\|_2^2.
\]
This gives the score-norm estimator
\[
\widehat w_{\mathrm{score}}
=
r\sqrt{\widehat \tau_n}.
\]
It requires neither matrix formation, spectral factorization, nor Monte Carlo
integration, and its cost is \(O(nd)\).

By Theorem~\ref{thm:fisher_ball_width},
\[
\sqrt{\frac{2}{\pi}}\,
r\sqrt{\Tr(G)}
\le
w_G(rB_2^d)
\le
r\sqrt{\Tr(G)}.
\]
Since
\[
\widehat \tau_n
\to
\Tr(G)
\]
almost surely by the law of large numbers, the estimator
\(\widehat w_{\mathrm{score}}\) converges almost surely to the upper scale
\(r\sqrt{\Tr(G)}\). Hence it provides a constant-factor approximation of
\(w_G(rB_2^d)\), with sharp universal factor \(\sqrt{2/\pi}\).

The constant is sharp: the lower factor \(\sqrt{2/\pi}\) is attained in the
rank-one case. In the isotropic case \(G=\lambda I_d\),
\[
\frac{w_G(rB_2^d)}{r\sqrt{\Tr(G)}}
=
\frac{\mathbb E\|g\|_2}{\sqrt d}
\to 1
\qquad
\text{as } d\to\infty.
\]

Using the decomposition
\[
\Tr(G)
=
\lambda_{\max}(G)\,d_F(G),
\]
the score-norm estimator may equivalently be viewed as estimating the scale
\[
r\sqrt{\lambda_{\max}(G)d_F(G)}.
\]

\begin{table}[h]
\centering
\caption{Comparison of Fisher width estimators.}
\label{tab:estimators}
\small
\begin{tabular}{llll}
\toprule
Method & Cost & Guarantee & Regime \\
\midrule
Full Fisher + MC
  & \(O(nd^2+d^3)\)
  & MC error \(+\varepsilon_n(\delta)w(T)\)
  & Moderate \(d\) \\
Low-rank Fisher
  & \(O(ndk+dk^2)\)
  & \((\sqrt{\lambda_{k+1}}+\varepsilon_n(\delta))w(T)\)
  & Fast spectral decay \\
Score norm
  & \(O(nd)\)
  & Constant-factor for \(T=rB_2^d\)
  & Euclidean balls \\
\bottomrule
\end{tabular}
\end{table}

These comparisons show that Fisher width is not only a theoretical complexity measure, but also a computable quantity whose approximation can be adapted to the available spectral and geometric structure.

\section{Experiments}
\label{sec:experiments}

The experiments provide empirical support for three aspects of the theory
across three model classes and two tasks.
Section~\ref{subsec:exp-trained} shows that Fisher width is a nontrivial,
computable quantity on trained models.
Section~\ref{subsec:exp-estimator} evaluates the estimation methods and
stability bounds of Section~\ref{sec:computation}.
Section~\ref{subsec:exp-generalization} examines the predicted
\(O(1/\sqrt n)\) scaling from Theorem~\ref{thm:generalization}.
All experiments use \(T=B_2^p\), are averaged over 10 independent random
seeds, and report error bars of one standard deviation.
Code is provided in the supplementary material.

\subsection{Experimental Setup}
\label{subsec:exp-setup}

\paragraph{Dataset.}
All experiments use MNIST~\citep{LeCun1998}.
Sections~\ref{subsec:exp-trained} and~\ref{subsec:exp-estimator}
use the binary subset (digits 0 vs.\ 1, \(13{,}007\) samples total).
Section~\ref{subsec:exp-generalization} uses all 10 classes
(\(70{,}000\) samples).
Features are standardized to zero mean and unit variance.

\paragraph{Models and Fisher estimation.}
Three model classes are used throughout.

\emph{Model~A: Binary logistic regression} (\(p=784\)).\\
The full empirical Fisher
\[
\widehat G_n
=
\frac1n\sum_{i=1}^n
s_{\hat\theta}(x_i)s_{\hat\theta}(x_i)^\top
\]
is computed exactly. Fisher width is estimated by Monte Carlo with
\(B=5{,}000\) samples.

\emph{Model~B: Softmax regression, 10 classes} (\(p=7{,}840\)).\\
The diagonal empirical Fisher
\[
\widehat\gamma_{k,j}
=
\frac1n\sum_{i=1}^n
\widehat p_{ik}(1-\widehat p_{ik})x_{ij}^2
\]
is used. Fisher width is estimated with \(B=10{,}000\) Monte Carlo samples.

\emph{Model~C: Ridge regression} (\(p=784\)).\\
The Fisher matrix
\[
G(\theta)=\sigma^{-2}\Sigma_X,
\qquad
\Sigma_X=\frac1nX^\top X,
\]
does not depend on \(\theta\), providing an analytic ground truth for Fisher
width via the eigenvalues of \(\Sigma_X\).

\paragraph{Regularization.}
We vary \(L_2\) regularization
\[
\lambda\in\{0,10^{-4},10^{-3},10^{-2},10^{-1},1,5\}
\]
in Sections~\ref{subsec:exp-trained}--\ref{subsec:exp-estimator}, and
\[
\lambda\in\{10^{-3},10^{-2},10^{-1},1\}
\]
in Section~\ref{subsec:exp-generalization}.

\subsection{Fisher Width in Trained Models}
\label{subsec:exp-trained}

\paragraph{Goal.}
The first experiment examines whether empirical Fisher width is substantially
different from the Euclidean baseline \(w(B_2^p)\approx\sqrt p\), whether it
satisfies the theoretical spectral and score bounds, and how it changes under
regularization across the three model classes.

\paragraph{Setup.}
We use \(n=10{,}000\), 10 seeds, and
\[
\lambda\in\{0,10^{-4},10^{-3},10^{-2},10^{-1},1,5\}.
\]

\paragraph{Observations.}

\emph{Fisher width is substantially below the Euclidean baseline.}
Figure~\ref{fig:exp1} and Table~\ref{tab:exp1} report the results.
For Model~A, the normalized ratio
\[
\widehat w_G(B_2^{784})/w(B_2^{784})
\]
ranges from \(0.001\) at \(\lambda=0\) to \(0.295\) at \(\lambda=5\).
For Model~B, the ratio ranges from \(0.030\) to \(0.266\) over the larger
parameter space \(p=7{,}840\).
For Model~C, the ratio is approximately \(0.865\), reflecting the
near-isotropic Fisher geometry of the Gaussian covariate model.

\emph{The theoretical bounds are satisfied.}
The spectral sandwich bound of Lemma~\ref{lem:spectral-bounds} holds across
all 70 runs for each model, corresponding to 7 regularization values and
10 random seeds. For Model~A,
\[
\min_{\lambda,s}
\left(
\widehat w_G^{(s)}
-
\sqrt{\lambda_{\min}(\widehat G)}\,w(B_2^p)
\right)
=
0.021>0.
\]
The score upper bound of Theorem ~\ref{thm:fisher_ball_width} is satisfied in every
run for Models~A and~C, with
\[
\min
\left(
\sqrt{\Tr(\widehat G)}-\widehat w_G
\right)
=
2.8\times10^{-4}.
\]

\emph{For the models considered here, Fisher width increases with
regularization.}
For all three models, \(\widehat w_G\) is increasing in \(\lambda\) across all
seeds; see Figure~\ref{fig:exp1}(a),(b).
In Model~A, the damped condition number
\(\kappa_\varepsilon(\widehat G)\) also increases with \(\lambda\), growing
from approximately \(10^2\) to \(10^7\). This appears to occur because the
fitted model changes the Fisher spectrum by increasing the leading
eigenvalues while the smallest eigenvalues remain close to zero.

\emph{The score upper bound is tight for near-isotropic models.}
The ratio between the score upper scale and the empirical Fisher width is
\(1.008\) for Model~A and \(1.006\) for Model~C, consistent with near-isotropic
Fisher geometry. The ratio is larger for Model~B.

\emph{The Monte Carlo estimator is validated against analytic ground truth in
Model~C.}
For ridge regression, the analytic Fisher width
\[
w_G^{\mathrm{analytic}}=24.38
\]
is available from the eigenvalues of \(\Sigma_X\). The Monte Carlo estimator
with \(B=5{,}000\) achieves a relative error of \(0.13\%\), uniformly across
all regularization values, directly supporting the estimation procedure of
Section~\ref{sec:computation}.

\begin{table}[htbp]
\centering
\small
\caption{%
Selected Fisher width values from Experiment~1.
\(w(B_2^p)=27.99\) for \(p=784\) and \(88.54\) for \(p=7{,}840\).
Means over 10 seeds; \(\pm\) std shown.%
}
\label{tab:exp1}
\setlength{\tabcolsep}{4pt}
\begin{tabular}{llcccc}
\toprule
Model & \(\lambda\)
  & \(\widehat w_G\)
  & \(\widehat w_G/w(B_2^p)\)
  & Score UB
  & \(\kappa_\varepsilon(\widehat G)\) \\
\midrule
\multirow{4}{*}{\shortstack[l]{A: Logistic\\(\(p=784\))}}
  & \(0\)        & \(0.028\pm0.005\) & \(0.001\) & \(0.029\) & \(2.2\times10^2\) \\
  & \(10^{-2}\) & \(1.462\pm0.039\) & \(0.052\) & \(1.469\) & \(1.7\times10^5\) \\
  & \(1\)       & \(5.623\pm0.061\) & \(0.201\) & \(5.650\) & \(2.9\times10^6\) \\
  & \(5\)       & \(8.253\pm0.056\) & \(0.295\) & \(8.304\) & \(8.5\times10^6\) \\
\midrule
\multirow{3}{*}{\shortstack[l]{B: Softmax\\(\(p=7840\))}}
  & \(0\)        & \(2.640\pm0.297\)  & \(0.030\) & \(2.641\)  & \(7.8\times10^4\) \\
  & \(10^{-1}\) & \(12.991\pm0.099\) & \(0.147\) & \(12.992\) & \(2.1\times10^5\) \\
  & \(5\)       & \(23.524\pm0.085\) & \(0.266\) & \(23.525\) & \(2.5\times10^5\) \\
\midrule
\multirow{2}{*}{\shortstack[l]{C: Ridge\\(\(p=784\))}}
  & all \(\lambda\) & \(24.212\pm0.075\) & \(0.865\) & \(24.450\)
  & \(1.0\times10^8\) \\
  & \multicolumn{5}{l}{%
    \footnotesize Analytic \(w_G=24.38\);\quad
    MC error \(=0.13\%\);\quad
    \(G=\sigma^{-2}\Sigma_X\) invariant to \(\lambda\).} \\
\bottomrule
\end{tabular}
\end{table}

\begin{figure}[tbp]
  \centering
  \includegraphics[width=\textwidth]{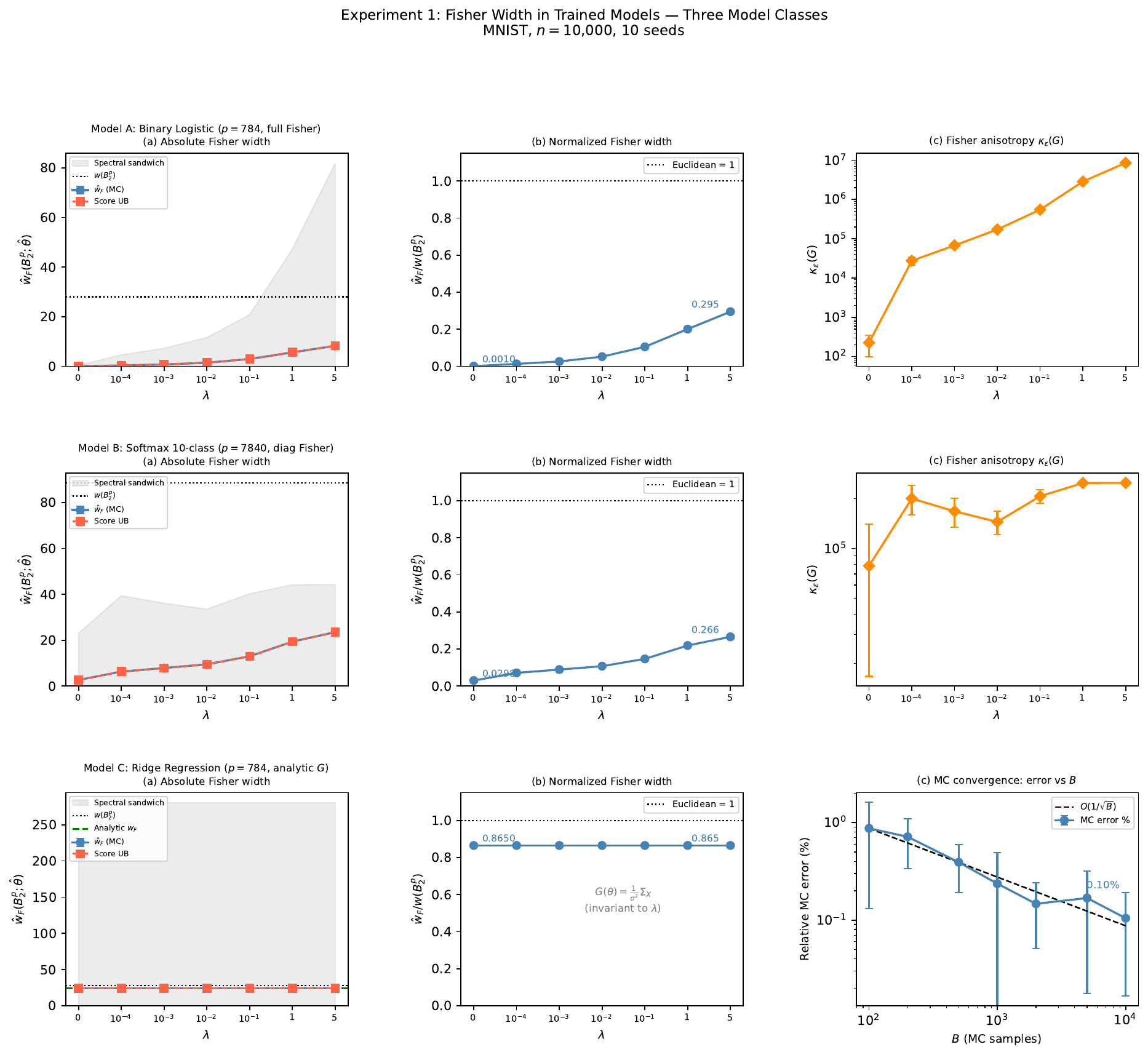}
  \caption{%
    \textbf{Experiment 1: Fisher width in trained models.}
    Three model classes on MNIST, \(n=10{,}000\), 10 seeds,
    \(\lambda\in\{0,10^{-4},10^{-3},10^{-2},10^{-1},1,5\}\).
    \emph{Row~A}: Binary logistic regression (\(p=784\), full Fisher).
    \emph{Row~B}: Softmax regression (\(p=7{,}840\), diagonal Fisher).
    \emph{Row~C}: Ridge regression (\(p=784\), analytic Fisher).
    \textbf{Left column}: empirical Fisher width \(\widehat w_G\) with score
    upper bound \(\sqrt{\Tr(\widehat G)}\) and spectral sandwich.
    The dotted line is the Euclidean baseline \(w(B_2^p)\).
    \textbf{Middle column}: normalized ratio
    \(\widehat w_G/w(B_2^p)\).
    \textbf{Right column}: damped condition number
    \(\kappa_\varepsilon(\widehat G)\) on a log scale for Models~A and~B;
    Monte Carlo convergence against analytic ground truth for Model~C,
    confirming the \(O(1/\sqrt B)\) rate.%
  }
  \label{fig:exp1}
\end{figure}
\FloatBarrier

\subsection{Estimator Accuracy and Stability}
\label{subsec:exp-estimator}

\paragraph{Goal.}
The second experiment evaluates the low-rank approximation error from
Proposition~\ref{prop:low-rank-computation}, the data convergence predicted
by Corollary~\ref{cor:empirical-fisher-stability}, and the stability of
structured approximations given by Theorem~\ref{thm:estimation-stability}.

\paragraph{Setup.}
We use Model~A with \(\lambda=0.01\), \(n=10{,}000\), and 10 seeds.
The reference value is
\[
\widehat w_G^{\mathrm{ref}}
=
1.462\pm0.039
\]
computed with \(B=10{,}000\) Monte Carlo samples.

\paragraph{Observations.}

\emph{Rank-\(k\) error tracks \(\sqrt{\lambda_{k+1}}\).}
Figure~\ref{fig:exp2}(a) reports
\[
\left|
\widehat w_G-\widehat w_{G,k}
\right|
\]
for
\[
k\in\{1,2,5,10,20,30,50,100,200,392\}.
\]
The bound of Proposition~\ref{prop:low-rank-computation} is satisfied at every
\(k\), with
\[
\max_k(\text{error}-\text{bound})=-0.262<0.
\]
On the log-log scale, the actual error is parallel to the bound across all
\(k\), supporting the predicted \(\Theta(\sqrt{\lambda_{k+1}})\) scaling.
The Fisher spectrum decays rapidly: by \(k=100\), the rank-\(k\) error is
\(0.096\), and by \(k=392\), it reaches \(0.033\).

\emph{Data convergence is consistent with an \(O(1/\sqrt n)\) rate.}
Figure~\ref{fig:exp2}(b) shows the error
\[
\left|
\widehat w_G^{(n)}-\widehat w_G^{\mathrm{ref}}
\right|
\]
as a function of \(n\). The log-log slope is \(-0.60\), close to the
\(-0.5\) rate predicted by Corollary~\ref{cor:empirical-fisher-stability}.
Relative error falls from \(47.7\%\) at \(n=100\) to \(2.0\%\) at
\(n=10{,}000\approx12.8d\). Errors are larger for \(n<d\), where the
empirical Fisher is severely rank-deficient.

\emph{The diagonal approximation is accurate in this setting.}
Figure~\ref{fig:exp2}(c) compares the full Fisher, diagonal Fisher, and score
upper bound. The diagonal Fisher incurs \(0.56\%\) relative error and the
score upper bound incurs \(0.64\%\) relative error relative to the full Fisher
width. Both approximations consistently upper-bound the full estimate across
all seeds, as expected from the stability principle.

\begin{figure}[tbp]
  \centering
  \includegraphics[width=\textwidth]{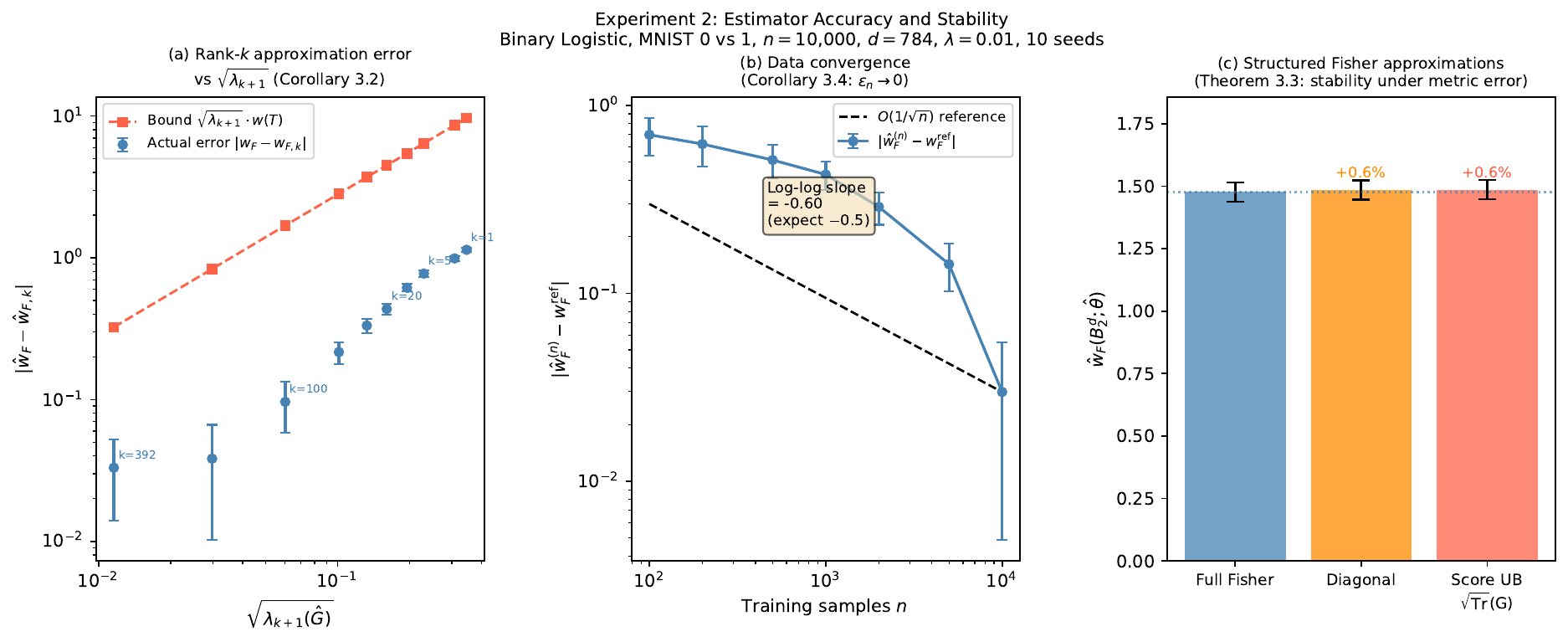}
  \caption{%
    \textbf{Experiment 2: Estimator accuracy and stability.}
    Model~A (binary logistic, \(\lambda=0.01\), \(n=10{,}000\),
    \(d=784\), 10 seeds).
    \textbf{(a)} Rank-\(k\) approximation error
    \(|\widehat w_G-\widehat w_{G,k}|\) versus the theoretical bound
    \(\sqrt{\lambda_{k+1}(\widehat G)}\,w(T)\) on a log-log scale.
    The bound is satisfied at all \(k\), with parallel scaling
    \(\Theta(\sqrt{\lambda_{k+1}})\)
    (Proposition~\ref{prop:low-rank-computation}).
    \textbf{(b)} Data convergence:
    \(|\widehat w_G^{(n)}-\widehat w_G^{\mathrm{ref}}|\) versus \(n\),
    with an \(O(1/\sqrt n)\) reference line;
    log-log slope \(=-0.60\)
    (Corollary~\ref{cor:empirical-fisher-stability}).
    \textbf{(c)} Structured approximations: full Fisher, diagonal Fisher
    (\(+0.56\%\)), and score upper bound (\(+0.64\%\))
    (Theorem~\ref{thm:estimation-stability}).%
  }
  \label{fig:exp2}
\end{figure}
\FloatBarrier

\subsection{Fisher Width and Generalization}
\label{subsec:exp-generalization}

\paragraph{Goal.}
The third experiment examines the \(O(1/\sqrt n)\) scaling predicted by
Theorem~\ref{thm:generalization} across training sizes and regularization
levels.

\paragraph{Setup.}
We use Model~B, the 10-class softmax model, with
\(N_{\mathrm{test}}=5{,}000\) and 10 seeds. We vary
\[
\lambda\in\{10^{-3},10^{-2},10^{-1},1\}
\]
and
\[
n\in\{200,500,1{,}000,2{,}000,5{,}000,10{,}000,20{,}000\}.
\]

Although the implementation uses the full \(K\times d\) parametrization
(\(p=7{,}840\)), the softmax model is identifiable only modulo common shifts
of all class weights. The Fisher-Lipschitz verification in
Appendix~\ref{app:softmax-fisher-lipschitz} is stated in an identifiable
baseline parametrization; the full-parametrization case is interpreted through
the quotient-space formulation of Remark~\ref{rem:full-softmax-parametrization}.

\paragraph{Observations.}

\emph{The generalization gap scales approximately as \(1/\sqrt n\).}
For each fixed \(\lambda\), the generalization gap is approximately linear in
\(1/\sqrt n\) across all seven training sizes; see
Figure~\ref{fig:exp3}(a). The linear fits achieve
\[
R^2\in\{0.935,0.943,0.987,0.976\}
\]
for
\[
\lambda\in\{10^{-3},10^{-2},10^{-1},1\},
\]
respectively, consistent with the rate predicted by
Theorem~\ref{thm:generalization}.

\emph{Fisher width increases with training data and is stabilized by
regularization.}
For fixed \(\lambda\), \(\widehat w_G(B_2^p;\hat\theta_n)\) grows with \(n\),
as shown in Figure~\ref{fig:exp3}(b). Stronger regularization stabilizes this
growth: the ratio
\[
\frac{\widehat w_G(n=20{,}000)}{\widehat w_G(n=200)}
\]
is \(5.97\) at \(\lambda=10^{-3}\) and \(1.19\) at \(\lambda=1\).

\emph{The observed gaps are consistent with Fisher-width scaling.}
Figure~\ref{fig:exp3}(c) plots the generalization gap against
\[
\widehat w_G/\sqrt n
\]
over all \((\lambda,n)\) pairs. All 28 averaged data points lie below a fitted
linear envelope
\[
CL\cdot \widehat w_G/\sqrt n
\]
with \(CL=2.672\). This supports the scaling predicted by
Theorem~\ref{thm:generalization}, although the constant is calibrated
empirically.

\begin{figure}[tbp]
  \centering
  \includegraphics[width=\textwidth]{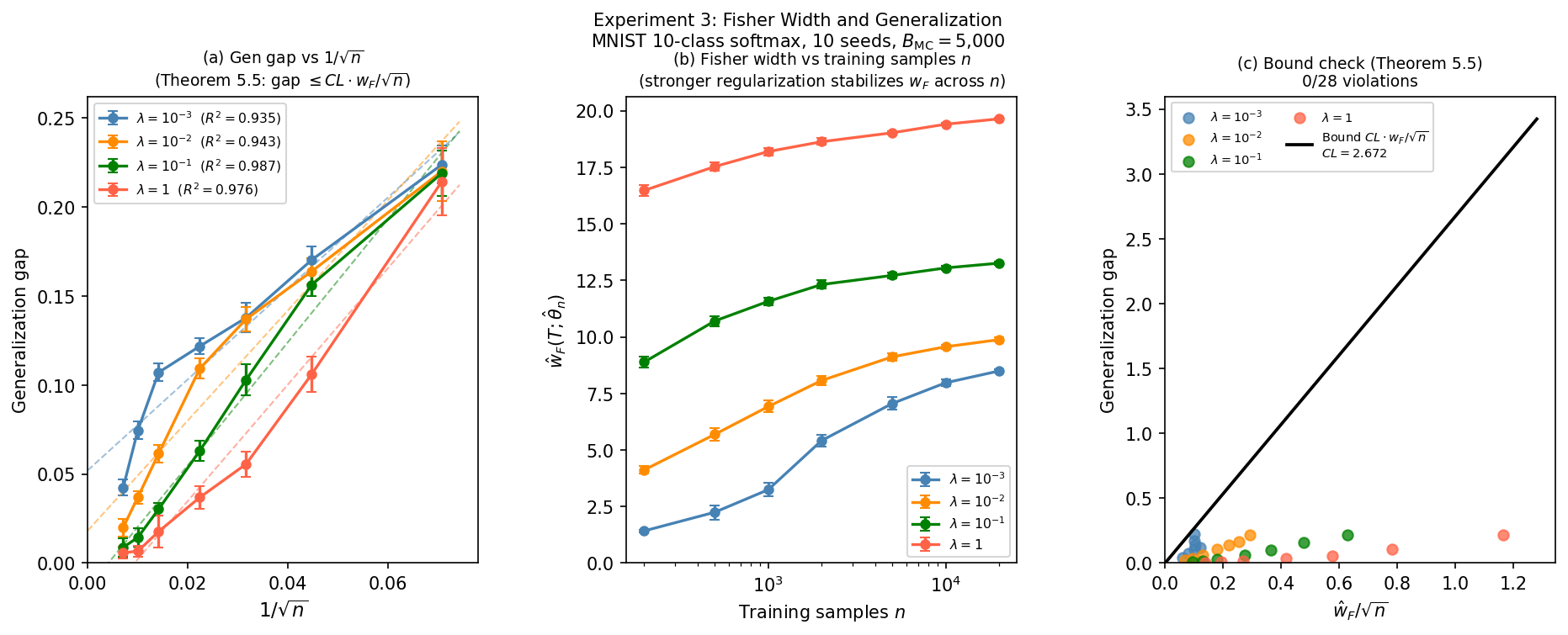}
  \caption{%
    \textbf{Experiment 3: Fisher width and generalization.}
    Model~B (MNIST 10-class softmax), 10 seeds,
    \(N_{\mathrm{test}}=5{,}000\).
    \textbf{(a)} Generalization gap versus \(1/\sqrt n\) for four
    regularization levels; dashed lines are linear fits.
    All four fits achieve \(R^2\ge0.935\), consistent with the
    \(O(1/\sqrt n)\) rate of Theorem~\ref{thm:generalization}.
    \textbf{(b)} Fisher width \(\widehat w_G\) versus training set size
    \(n\) on a log scale for fixed \(\lambda\).
    Stronger regularization stabilizes \(\widehat w_G\) across \(n\).
    \textbf{(c)} Generalization gap versus
    \(\widehat w_G/\sqrt n\) for all \((\lambda,n)\) pairs.
    All 28 averaged points lie below the fitted envelope
    \(CL\cdot \widehat w_G/\sqrt n\) with \(CL=2.672\).%
  }
  \label{fig:exp3}
\end{figure}
\FloatBarrier

\section{Discussion}
\label{sec:discussion}

\subsection{Complexity Depends on Geometry}

A central message of this paper is that complexity is not intrinsic to a hypothesis class: it depends on the metric used to measure it.

Classical Gaussian width \(w(T)\) implicitly assumes the Euclidean metric \(G=I\), treating all parameter directions as equally significant. Fisher width
\[
w_G(T)=w(G^{1/2}T)
\]
replaces this assumption with the local statistical geometry of the model: directions of high Fisher curvature are amplified, statistically insensitive directions are suppressed, and the resulting complexity measure reflects what is locally distinguishable from data.

This perspective has a concrete consequence. Two hypothesis classes with identical Euclidean complexity may have very different Fisher complexities if their Fisher geometries differ. It is Fisher complexity, rather than Euclidean
complexity alone, that appears in the generalization bound of Theorem~\ref{thm:generalization}.

Gaussian width is therefore not replaced by Fisher width; it appears as the flat, isotropic limit of a more general information-geometric complexity viewpoint.

\subsection{Effective Fisher Dimension}

A central consequence of Fisher width is the emergence of an intrinsic effective dimension determined by the Fisher geometry itself. The quantity
\[
d_F(G)
=
\frac{\Tr(G)}{\lambda_{\max}(G)}
\]
appears naturally from the sharp characterization
\[
w_G(rB_2^d)^2
\asymp
r^2\lambda_{\max}(G)d_F(G)
\]
of Theorem~\ref{thm:fisher_ball_width}. For nonzero \(G\),
\[
1\le d_F(G)\le \operatorname{rank}(G)\le d.
\]
Thus \(d_F(G)\) measures how many directions carry substantial Fisher information, interpolating between the rank-one, maximally anisotropic case and the isotropic case \(G=\lambda I_d\), where \(d_F(G)=d\).

The significance of \(d_F(G)\) is that it can replace raw parameter count as a more intrinsic notion of local model size. For Euclidean balls with comparable radius and largest Fisher scale, a model with \(d=10^6\) but \(d_F(G)=50\) has the same Fisher-width scaling as a 50-dimensional isotropic model.
The experiments of Section~\ref{sec:experiments} show that Fisher geometry varies substantially across model classes and regularization levels, suggesting that \(d_F(G)\) captures structure beyond the raw parameter count
\(d\). A systematic comparison between \(d_F(G)\) and \(d\) as predictors of generalization remains for future work.

\subsection{Limitations and Future Directions}

Several limitations bound the scope of the present results.

\textbf{Locality.}
Fisher width is defined at a fixed base point \(\theta_0\). Extending the theory to globally varying Fisher metrics on overparameterized manifolds remains open. Such an extension would require controlling how \(G(\theta)\), and therefore \(w_{G(\theta)}(T)\), changes along the manifold.

\textbf{Static metric.}
The theory assumes \(G(\theta_0)\) fixed. In modern learning systems, the Fisher geometry evolves throughout training. Whether the static bounds of this paper remain informative along training trajectories is not yet understood.
This motivates a dynamic theory of Fisher complexity, in which Fisher width is studied as a time-dependent geometric quantity along optimization paths.

\textbf{Perturbative curvature.}
The curvature analysis of Section~\ref{sec:curvature} is local and perturbative. A global geometric characterization of Fisher complexity on curved statistical manifolds remains an open problem.

\textbf{Sharpness beyond exponential families.}
Section~\ref{subsec:tightness} shows that the Fisher-width scale is sharp for local perturbations in exponential families with negative log-likelihood loss.
Whether a matching lower bound holds for general Fisher-Lipschitz losses under a suitable nondegeneracy condition remains open. Establishing such a result would confirm that Fisher width is not merely sufficient, but also necessary,
as a complexity measure for statistical learning on manifolds.

\textbf{Fisher-Lipschitz assumption.}
The generalization bound relies on Assumption~\ref{ass:fisher-lipschitz}.
Appendix~\ref{app:fisher-lipschitz} verifies this assumption for logistic and softmax regression under bounded Fisher-leverage conditions. Extending such
verification to neural networks and other highly nonconvex model classes remains open.

\textbf{Population versus empirical Fisher.}
The theory is formulated in terms of the population Fisher matrix \(G(\theta_0)\), whereas practical estimation relies on the empirical approximation \(\widehat G_n\). Section~\ref{sec:computation} quantifies the resulting width error, but understanding the full statistical and algorithmic consequences of this replacement, particularly in high-dimensional, misspecified, or overparameterized settings, remains an important direction for future work.

\textbf{Recovery and duality.}
Finally, Fisher width suggests a possible primal--dual extension of the present theory. By analogy with Gaussian width in conic integral geometry, a dual Fisher width such as \(w_{G^{-1}}(T)\), or \(w_{G^\dagger}(T)\) in
singular settings, may be relevant for Fisher-scaled recovery thresholds. This points toward a picture in which \(w_G\) measures statistical sensitivity, while a dual width measures estimation uncertainty or recovery difficulty. We leave this direction for future work. 

\begin{appendices}

\section{Verification of the Fisher-Lipschitz Assumption}
\label{app:fisher-lipschitz}

This appendix verifies Assumption~\ref{ass:fisher-lipschitz} for the
classification models used in the experiments. The common mechanism is a
Fisher-gradient bound: if the loss gradient has uniformly bounded dual norm
with respect to the base-point Fisher metric, then the loss is Lipschitz in
the corresponding Fisher norm.

\subsection{A General Fisher-Gradient Criterion}
\label{app:fisher-gradient-criterion}

\begin{lemma}[Fisher-gradient criterion]
\label{lem:fisher-gradient-criterion}
Let \(T\subset\mathbb R^d\) be convex and let \(G_0\succ0\). Suppose that
\(\ell(\theta;z)\) is differentiable in \(\theta\) on \(T\), and that
\[
\sup_{\theta\in T}\sup_z
\left\|
G_0^{-1/2}\nabla_\theta\ell(\theta;z)
\right\|_2
\le
L.
\]
Then \(\ell\) is Fisher-Lipschitz on \(T\) with respect to \(G_0\):
\[
|\ell(\theta_1;z)-\ell(\theta_2;z)|
\le
L\|\theta_1-\theta_2\|_{G_0},
\qquad
\theta_1,\theta_2\in T.
\]
\end{lemma}

\begin{proof}
Fix \(\theta_1,\theta_2\in T\), and define
\[
\theta_t
=
\theta_2+t(\theta_1-\theta_2),
\qquad
t\in[0,1].
\]
By convexity of \(T\), \(\theta_t\in T\). The fundamental theorem of calculus
gives
\[
\ell(\theta_1;z)-\ell(\theta_2;z)
=
\int_0^1
\left\langle
\nabla_\theta\ell(\theta_t;z),
\theta_1-\theta_2
\right\rangle
\,dt.
\]
For each \(t\), Cauchy--Schwarz in the Fisher metric gives
\[
\begin{aligned}
\left|
\left\langle
\nabla_\theta\ell(\theta_t;z),
\theta_1-\theta_2
\right\rangle
\right|
&=
\left|
\left\langle
G_0^{-1/2}\nabla_\theta\ell(\theta_t;z),
G_0^{1/2}(\theta_1-\theta_2)
\right\rangle
\right|  \\
&\le
\left\|
G_0^{-1/2}\nabla_\theta\ell(\theta_t;z)
\right\|_2
\,
\|\theta_1-\theta_2\|_{G_0}.
\end{aligned}
\]
Using the assumed uniform bound and integrating over \(t\), we obtain
\[
|\ell(\theta_1;z)-\ell(\theta_2;z)|
\le
L\|\theta_1-\theta_2\|_{G_0}.
\]
\end{proof}

\begin{remark}
The quantity
\[
\left\|
G_0^{-1/2}\nabla_\theta\ell(\theta;z)
\right\|_2
\]
is the dual norm of the loss gradient with respect to the base-point Fisher
metric. Thus Lemma~\ref{lem:fisher-gradient-criterion} reduces verification
of Assumption~\ref{ass:fisher-lipschitz} to a uniform Fisher-leverage bound.
\end{remark}

\subsection{Logistic Regression}
\label{app:logistic-fisher-lipschitz}

\begin{proposition}[Fisher-Lipschitz condition for logistic regression]
\label{prop:logistic-fisher-lipschitz}
Consider binary logistic regression
\[
\mathbb P_\theta(Y=1\mid X=x)
=
\sigma(\theta^\top x),
\qquad
\sigma(t)=\frac{1}{1+e^{-t}},
\]
with cross-entropy loss
\[
\ell(\theta;(x,y))
=
-y\,\theta^\top x+\log(1+e^{\theta^\top x}).
\]
Fix a base point \(\theta_0\), and let \(G_0=G(\theta_0)\succ0\). Suppose that
\(T\subset\mathbb R^d\) is convex and
\[
\sup_x
\|G_0^{-1/2}x\|_2
\le
L.
\]
Then Assumption~\ref{ass:fisher-lipschitz} holds on \(T\) with constant \(L\).
\end{proposition}

\begin{proof}
For logistic regression,
\[
\nabla_\theta\ell(\theta;(x,y))
=
(\sigma(\theta^\top x)-y)x.
\]
Since \(y\in\{0,1\}\) and \(\sigma(\theta^\top x)\in[0,1]\),
\[
|\sigma(\theta^\top x)-y|
\le
1.
\]
Therefore
\[
\left\|
G_0^{-1/2}\nabla_\theta\ell(\theta;(x,y))
\right\|_2
\le
\|G_0^{-1/2}x\|_2
\le
L.
\]
The claim follows from Lemma~\ref{lem:fisher-gradient-criterion}.
\end{proof}

\subsection{Softmax Regression}
\label{app:softmax-fisher-lipschitz}

We next consider \(K\)-class softmax regression. To avoid the non-identifiability
of the full \(K\times d\) parametrization, we use the standard baseline
parametrization \(W\in\mathbb R^{(K-1)\times d}\), where class \(K\) is the
baseline.

For \(x\in\mathbb R^d\), define
\[
p_W(k\mid x)
=
\frac{\exp(w_k^\top x)}
{1+\sum_{\ell=1}^{K-1}\exp(w_\ell^\top x)},
\qquad
k=1,\ldots,K-1,
\]
and
\[
p_W(K\mid x)
=
\frac{1}
{1+\sum_{\ell=1}^{K-1}\exp(w_\ell^\top x)}.
\]
Write
\[
\pi_W(x)
=
(p_W(1\mid x),\ldots,p_W(K-1\mid x))
\in\mathbb R^{K-1}.
\]
The cross-entropy loss is
\[
\ell(W;(x,y))
=
-\log p_W(y\mid x).
\]

\begin{proposition}[Fisher-Lipschitz condition for softmax regression]
\label{prop:softmax-fisher-lipschitz}
Fix a base point \(W_0\), and let \(G_0=G(W_0)\succ0\) be the Fisher matrix in
the identifiable baseline parametrization. Suppose \(T\subset\mathbb R^{(K-1)\times d}\)
is convex.

If
\[
\sup_{W\in T}\sup_{(x,y)}
\left\|
G_0^{-1/2}\nabla_W\ell(W;(x,y))
\right\|_F
\le
L,
\]
then Assumption~\ref{ass:fisher-lipschitz} holds on \(T\) with constant \(L\).

In particular, if \(\|x\|_2\le R\) almost surely and
\[
G_0\succeq \mu I
\]
for some \(\mu>0\), then Assumption~\ref{ass:fisher-lipschitz} holds with
\[
L
=
\frac{\sqrt2\,R}{\sqrt\mu}.
\]
\end{proposition}

\begin{proof}
The first statement follows directly from Lemma~\ref{lem:fisher-gradient-criterion},
identifying \(W\) with its vectorization and using the Frobenius inner product.

It remains to prove the explicit bound. Differentiating the cross-entropy loss gives
\[
\nabla_W\ell(W;(x,y))
=
(\pi_W(x)-e_y)x^\top,
\]
where \(e_y\in\mathbb R^{K-1}\) is the \(y\)-th standard basis vector for
\(y\in\{1,\ldots,K-1\}\), and \(e_K=0\) for the baseline class.

Since this is a rank-one matrix,
\[
\|\nabla_W\ell(W;(x,y))\|_F
=
\|\pi_W(x)-e_y\|_2\,\|x\|_2.
\]
Let
\[
v=\pi_W(x)-e_y.
\]
Then \(\|v\|_\infty\le1\). Moreover,
\[
\|v\|_1
\le
\|\pi_W(x)\|_1+\|e_y\|_1
\le
1+1
=
2,
\]
because \(\sum_{k=1}^{K-1}p_W(k\mid x)\le1\) and \(\|e_y\|_1\le1\). Hence
\[
\|v\|_2^2
\le
\|v\|_\infty\|v\|_1
\le
2,
\]
so
\[
\|\pi_W(x)-e_y\|_2
\le
\sqrt2.
\]
Therefore, if \(\|x\|_2\le R\),
\[
\|\nabla_W\ell(W;(x,y))\|_F
\le
\sqrt2\,R.
\]
Finally, since \(G_0\succeq\mu I\),
\[
\left\|
G_0^{-1/2}\nabla_W\ell(W;(x,y))
\right\|_F
\le
\frac{1}{\sqrt\mu}
\|\nabla_W\ell(W;(x,y))\|_F
\le
\frac{\sqrt2\,R}{\sqrt\mu}.
\]
The result follows from Lemma~\ref{lem:fisher-gradient-criterion}.
\end{proof}

\begin{remark}[Full softmax parametrization]
\label{rem:full-softmax-parametrization}
In the full parametrization \(W\in\mathbb R^{K\times d}\), the softmax model is
not identifiable: adding the same row vector to all class weights leaves the
probabilities unchanged. Consequently, the Fisher matrix is singular, with
null space
\[
\mathcal N
=
\{\mathbf 1_K\otimes v:\ v\in\mathbb R^d\}.
\]
The same argument applies after restricting parameter differences to the
identifiable subspace
\[
\mathcal N^\perp
=
\left\{
W\in\mathbb R^{K\times d}:
\sum_{k=1}^K w_k=0
\right\}.
\]
On this subspace, the Fisher seminorm
\[
\|B\|_{G_0}
=
\|G_0^{1/2}B\|_F
\]
is nondegenerate, while the dual norm of the gradient is
\[
\|\nabla_W\ell\|_{G_0^\dagger}
=
\|G_0^{\dagger 1/2}\nabla_W\ell\|_F.
\]
Thus the Fisher-Lipschitz condition holds provided
\[
\sup_{W\in T}\sup_z
\|\nabla_W\ell(W;z)\|_{G_0^\dagger}
\le
L
\]
and \(W_1-W_2\in\mathcal N^\perp\) for the parameter differences under
consideration.
\end{remark}


\section{Proofs for Section~\ref{subsec:tightness}}
\label{app:tightness}

\begin{proof}[Proof of Lemma~\ref{lem:exact-gap-expfam}]
  Since $-\log p_\theta(z) = A(\theta) - \theta^\top\phi(z) - \log h(z)$,
  \[
    \widetilde\ell(u;z)
    = \bigl(A(\theta_0+u) - A(\theta_0)\bigr) - u^\top\phi(z).
  \]
  The deterministic term $A(\theta_0+u)-A(\theta_0)$ cancels in the
  difference between population and empirical risk:
  \begin{align*}
    R(u) - \widehat{R}_n(u)
    &= \mathbb{E}_{p_{\theta_0}}[\widetilde\ell(u;Z)]
      - \frac{1}{n}\sum_{i=1}^n\widetilde\ell(u;Z_i) \\
    &= -u^\top\mathbb{E}_{p_{\theta_0}}[\phi(Z)]
      + u^\top\overline\phi_n
    = u^\top\Delta_n.
  \end{align*}
  Therefore
  \[
    \sup_{u\in rB_2^d}
    \bigl|R(u) - \widehat{R}_n(u)\bigr|
    = \sup_{\|u\|_2\le r}|u^\top\Delta_n|
    = r\|\Delta_n\|_2,
  \]
  where the last equality follows from
  \[
    \sup_{\|u\|_2\le r}|u^\top v| = r\|v\|_2.
  \]
\end{proof}

\begin{proof}[Proof of Theorem~\ref{thm:tightness-expfam}]
  By Lemma~\ref{lem:exact-gap-expfam},
  \[
    \sup_{u\in rB_2^d}
    \bigl|R(u)-\widehat{R}_n(u)\bigr|
    = r\|\Delta_n\|_2.
  \]
  Set $W_n := \sqrt{n}\,\Delta_n$. By the multivariate CLT,
  \[
    W_n \xrightarrow{d} W, \qquad W\sim\mathcal{N}(0,G_0).
  \]
  Since
  \[
    \mathbb{E}\|W_n\|_2^2
    = n\,\mathbb{E}\|\Delta_n\|_2^2
    = \operatorname{Tr}(G_0) < \infty,
  \]
  we have $\sup_{n\ge 1}\mathbb{E}\|W_n\|_2^2 = \operatorname{Tr}(G_0) < \infty$.
  Hence the family $\{\|W_n\|_2 : n\ge 1\}$ is uniformly integrable. Convergence in
  distribution together with uniform integrability gives
  \[
    \mathbb{E}\|W_n\|_2 \longrightarrow \mathbb{E}\|W\|_2
    = \mathbb{E}\|G_0^{1/2}g\|_2,
    \qquad g\sim\mathcal{N}(0,I_d).
  \]
  By the Euclidean-ball formula for Fisher width
  (Definition~\ref{def:fisher-width} and Example~\ref{ex:euclidean-ball}),
  \[
    w_{G_0}(rB_2^d) = r\,\mathbb{E}\|G_0^{1/2}g\|_2.
  \]
  Combining,
  \[
    \sqrt{n}\;
    \mathbb{E}\!\left[
      \sup_{u\in rB_2^d}|R(u)-\widehat{R}_n(u)|
    \right]
    = r\,\mathbb{E}\|W_n\|_2
    \;\longrightarrow\; w_{G_0}(rB_2^d).
  \]
  The finite-$n$ lower bound follows immediately from the convergence.
\end{proof}
\end{appendices}

\section*{Acknowledgments}

The author(s) acknowledge the use of ChatGPT and Claude to assist in the preparation of this manuscript. Specifically, AI tools were utilized to refine the language and structure of the drafting process, to brainstorm and explore strategies for mathematical proofs, and to assist in
generating code for empirical validation. The author(s) meticulously reviewed, rigorously verified all step-by-step mathematical logic, and thoroughly debugged the source code. The author(s) take full and sole responsibility for the originality, correctness, and final content of this work.

\bibliographystyle{plainnat}
\bibliography{fisher_width}

\end{document}